%% file: main.tex
\documentclass[lettersize,journal]{IEEEtran}
\usepackage{amsmath,amsfonts}
\usepackage{algorithmic}
\usepackage{array}
\usepackage[caption=false,font=normalsize,labelfont=sf,textfont=sf]{subfig}
\usepackage{textcomp}
\usepackage{stfloats}
\usepackage{url}
\usepackage{verbatim}
\usepackage{graphicx}

\def\BibTeX{{\rm B\kern-.05em{\sc i\kern-.025em b}\kern-.08em
    T\kern-.1667em\lower.7ex\hbox{E}\kern-.125emX}}
\usepackage{balance}

\usepackage{xcolor}
\definecolor{CColor}{rgb}{0.01,0.31,0.59}
\definecolor{GGray}{rgb}{0.80,0.90,1}
\definecolor{Shady}{rgb}{0.9,0.9,0.9}
\definecolor{kaistblue}{RGB}{20,135,200}
\definecolor{kaistdarkblue}{RGB}{0,65,145}
\definecolor{urbanablue}{RGB}{19,41,75}
\definecolor{urbanaorange}{RGB}{232,74,39}
\definecolor{drp}{rgb}{0.53,0.15,0.34}
\usepackage[plainpages=false,pdfpagelabels,colorlinks=true,linkcolor=CColor,citecolor=CColor]{hyperref}
 
\usepackage{graphicx}
\usepackage{amsmath}
\usepackage{amssymb}
\usepackage{mathtools}
\usepackage{amsthm}
\usepackage[capitalize,noabbrev]{cleveref}
\usepackage[textsize=tiny]{todonotes}
\usepackage{booktabs}
\usepackage{subcaption}

\theoremstyle{plain}

\newtheorem{theorem}{Theorem}

\newtheorem{lemma}[theorem]{Lemma}
\newtheorem{corollary}[theorem]{Corollary}
\theoremstyle{definition}
\newtheorem{definition}[theorem]{Definition}

\input{macros}

\input{fonts}

\input{packages}

\begin{document}

\title{Expressive Power of Floating-Point Neural Networks with\\ Arbitrary Reduction Orders and Inexact Activation Implementations}

\author{
Yeachan Park,
Geonho Hwang,
Wonyeol Lee,
and Sejun Park
\thanks{
Yeachan Park is with the Department of Mathematics and Statistics, Sejong University, Seoul 05006, Republic of Korea
(e-mail: ychpark@sejong.ac.kr).
}
\thanks{
Geonho Hwang is with the Department of Mathematical Sciences, GIST, Gwangju Institute of Science and Technology 61005, Republic of Korea
(e-mail: hgh2134@gist.ac.kr).
}
\thanks{
Wonyeol Lee is with the Department of Computer Science, POSTECH, Pohang 37673, (email: wonyeol.lee.cs@gmail.com).
}
\thanks{
Sejun Park is with the Department of Artificial Intelligence, Korea University, Seoul 02841, Republic of Korea
(e-mail: sejun.park000@gmail.com).
}
\thanks{Corresponding author: Sejun Park.}
}

\maketitle

\begin{abstract}
Most existing expressivity theories for neural networks assume exact real arithmetic, whereas practical neural networks are executed under finite-precision floating-point arithmetic with implementation-dependent execution semantics.
Recent works have begun studying the expressive power of floating-point neural networks, but existing results are limited to highly restricted activation functions and idealized assumptions such as fixed left-to-right reduction orders and correctly rounded activation implementations.
In this work, we study the expressive power of floating-point neural networks under generalized floating-point execution semantics, including arbitrary reduction orders and inexact activation implementations with bounded ulp errors.
We investigate when floating-point neural networks can represent arbitrary functions between floating-point domains exactly.
To this end, we introduce a general distinguishability framework and show that the ability to distinguish every pair of distinct inputs in the first layer is necessary for universal representability.
This characterization yields broad classes of activation implementations that are not universal representators, extending previous isolated counterexamples such as the correctly rounded cosine activation.
We further prove that a suitable form of distinguishability is also sufficient for universal representability under mild conditions on the activation implementation.
Using this framework, we establish universal representability results for a broad class of practical activation functions, including implementations of $\Sigmoid$, $\tanh$, $\relu$, $\elu$, $\mathrm{SeLU}$, $\GeLU$, $\Swish$, $\Mish$, and $\sin$, under significantly more realistic floating-point execution models than previously known.
\end{abstract}


\input{1-introduction}
\input{2-preliminary}

\input{3-main_results}

\input{4-proofs}

\bibliography{reference}
\bibliographystyle{plain}

\end{document}

%% file: macros.tex
\allowdisplaybreaks

\newcommand*{\commentout}[1]{}

\newlength{\parskiptrue}
\definecolor{lred}{rgb}{1.0, 0.5, 0.5}
\definecolor{lorange}{rgb}{1.00, 0.90, 0.20}
\definecolor{lgreen}{rgb}{0.35, 0.95, 0.35}
\definecolor{lime}{rgb}{0.9, 1.0, 0.6}
\definecolor{lblue}{rgb}{1.0, 0.85, 0.75}

\makeatletter

\newcommand*\wthelper[2]{%
        \hbox{\dimen@\accentfontxheight#1%
                \accentfontxheight#11.1\dimen@
                $\m@th#1\widetilde{#2}$%
                \accentfontxheight#1\dimen@
        }%
}
\newcommand*\accentfontxheight[1]{%
        \fontdimen5\ifx#1\displaystyle
                \textfont
        \else\ifx#1\textstyle
                \textfont
        \else\ifx#1\scriptstyle
                \scriptfont
        \else
                \scriptscriptfont
        \fi\fi\fi3
}

\newcommand*\whhelper[2]{%
        \hbox{\dimen@\accentfontxheight#1%
                \accentfontxheight#11.2\dimen@
                $\m@th#1\widehat{#2}$%
                \accentfontxheight#1\dimen@
        }%
}
\makeatother

\makeatletter
\newcommand{\oset}[3][0ex]{%
  \mathrel{\mathop{#3}\limits^{
    \vbox to#1{\kern-3\ex@
    \hbox{$\scriptstyle#2$}\vss}}}}
\makeatother

\newcommand*{\defeq}{\coloneq}

\newcommand*{\relu}{\mathrm{ReLU}}

\newcommand*{\SoftPlus}{\mathrm{SoftPlus}}
\newcommand*{\Mish}{\mathrm{Mish}}
\newcommand*{\GeLU}{\mathrm{GELU}}

\newcommand*{\Sigmoid}{\mathrm{Sigmoid}}
\newcommand*{\ELU}{\mathrm{ELU}}
\newcommand*{\Swish}{\mathrm{Swish}}
\newcommand*{\elu}{\mathrm{ELU}}

\DeclareMathOperator*{\argmin}{arg\,min}
\newcommand*{\fpq}{\mathbb{F}}
\newcommand*{\efpq}{{\overline{\mathbb{F}}}}
\newcommand*{\fmin}{\omega}
\newcommand*{\fmax}{\Omega}

\newcommand*{\round}[1]{ \left\lceil{#1}\right\rfloor }

\newcommand*{\ulp}{    \mathrm{ulp}}

\makeatletter

\newcommand*{\tranLR}{\xRightarrow{(\sigma, \overline{\mcS})}{}}

\newcommand{\xMapsto}[2][]{\ext@arrow 0599{\Mapstofill@}{#1}{#2}}
\def\Mapstofill@{\arrowfill@{\Mapstochar\Relbar}\Relbar\Rightarrow}
\makeatother
\makeatletter
\def\moverlay{\mathpalette\mov@rlay}
\def\mov@rlay#1#2{\leavevmode\vtop{%
   \baselineskip\z@skip \lineskiplimit-\maxdimen
   \ialign{\hfil$\m@th#1##$\hfil\cr#2\crcr}}}
\newcommand{\charfusion}[3][\mathord]{
    #1{\ifx#1\mathop\vphantom{#2}\fi
        \mathpalette\mov@rlay{#2\cr#3}
      }
    \ifx#1\mathop\expandafter\displaylimits\fi}
\makeatother

\newcommand*{\resig}{\rho}
\newcommand{\lrp}[1]{\left({#1}\right)}

\newcommand{\emin}{\mathfrak{e}_{\min}}
\newcommand{\emax}{\mathfrak {e}_{\max}}
\newcommand{\hatsig}{\hat{\sigma}}

\newcommand*{\depth}{\mathrm{depth}}

\newcommand*{\mant}[1]{\mathfrak {m}_{#1}}

\newcommand*{\expo}[1]{\mathfrak {e}_{#1}}
\newcommand*{\mantd}[2]{\mathfrak {m}_{#1,#2}}

\newcommand*{\nmant}[1]{\widetilde{\mathfrak {m}}_{#1}}
\newcommand*{\nexpo}[1]{\widetilde{\mathfrak {e}}_{#1}}
\newcommand{\nan}{\text{NaN}}

\newcommand*{\domain}{\mathcal{X}}

\newcommand{\mbit}{p}
\newcommand{\ebit}{q}

\newcommand*{\indcc}[1]{\mathbbm{1}_{#1}}
\newcommand*{\indc}[2]{\mathbbm{1}_{#1}\left({#2}\right)}

\DeclarePairedDelimiter\ceil{\lceil}{\rceil}

%% file: fonts.tex
\newcommand{\mathboldcommand}[1]{\mathbb{#1}}

\newcommand{\bbN}{\mathboldcommand{N}}

\newcommand{\bbR}{\mathboldcommand{R}}
\newcommand{\bbS}{\mathboldcommand{S}}

\newcommand{\bbZ}{\mathboldcommand{Z}}

\newcommand{\bfx}{\mathbf{x}}

\newcommand{\bfz}{\mathbf{z}}

\usepackage{euscript}
\newcommand{\mathcalcommand}[1]{\mathcal{#1}}

\newcommand{\mcI}{\mathcalcommand{I}}

\newcommand{\mcK}{\mathcalcommand{K}}

\newcommand{\mcM}{\mathcalcommand{M}}
\newcommand{\mcN}{\mathcalcommand{N}}
\newcommand{\mcO}{\mathcalcommand{O}}

\newcommand{\mcR}{\mathcalcommand{R}}
\newcommand{\mcS}{\mathcalcommand{S}}

\newcommand{\mcZ}{\mathcalcommand{Z}}

\DeclareMathAlphabet{\mathpzc}{T1}{pzc}{m}{it}

%% file: 1-introduction.tex
\section{Introduction}
\label{sec:intro}

Deep neural networks have achieved remarkable success across a wide range of scientific and engineering applications \cite{lecun2015deep}.
A central theoretical foundation behind this success is the universal approximation theorem, which states that neural networks can approximate broad classes of target functions.
Classical results show that shallow neural networks with non-polynomial activation functions can approximate arbitrary continuous functions on compact domains \cite{cybenko1989approximation,Hornik89,pinkus99,leshno1993multilayer}.
These results have since been extended to width-bounded architectures and modern neural network models \cite{Lu17,kidger2020universal,park21,zhou2020universality,tabuada2021universal,Yun2020Are,ramanujan2020s,yuan2023comprehensive}.

Most existing expressivity theories, however, analyze neural networks under exact real arithmetic.
In practice, neural networks are executed using finite-precision machine arithmetic, where parameters belong only to a finite set of machine-representable numbers and arithmetic operations incur rounding errors.
This discrepancy between theoretical assumptions and practical computation becomes particularly significant in low-precision settings such as float16, bfloat16, and float8, where both the number of representable values and the precision of arithmetic operations are severely limited.

Motivated by this gap between theory and practice, several recent works have begun studying neural networks under finite-precision computation.
For example, \cite{ding2018on} and \cite{gonon2023approximation} studied neural networks with quantized parameters under exact arithmetic.
More recent works investigated neural networks under inexact machine arithmetic.
\cite{hwang2024expressive} studied neural networks under fixed-point arithmetic and established universal approximation results for several practical activation functions.
Likewise, \cite{park2024expressive} analyzed floating-point neural networks under floating-point arithmetic and showed that networks using $\relu$ and $\mathrm{Step}$ activations can represent arbitrary floating-point functions over a bounded domain.

Despite these advances, existing analyses of floating-point neural networks remain limited in several important aspects.
First, previous floating-point representability results only cover highly restricted activation functions such as $\relu$ and $\mathrm{Step}$.
In particular, the analysis of \cite{park2024expressive} relies heavily on the piecewise-linearity of activation functions, making it unclear whether the theory extends to smooth and widely used activations such as $\Sigmoid$, $\tanh$, $\relu$, $\elu$, $\mathrm{SeLU}$, $\GeLU$, $\Swish$, $\Mish$, and $\sin$.
Second, previous works typically assume idealized execution semantics, including fixed left-to-right reduction orders and correctly rounded activation implementations.

However, these assumptions do not fully capture how neural networks are evaluated in practical floating-point environments.
Floating-point addition is non-associative, so the output of a network may depend on the order in which reductions are performed.
This phenomenon becomes particularly relevant in massively parallel systems such as GPUs, where reductions are often implemented through hardware-dependent evaluation orders rather than a fixed left-to-right order.
Furthermore, practical implementations of activation functions such as $\exp$, $\tanh$, $\sin$, and $\GeLU$ are typically not correctly rounded, and may differ from the correctly rounded values by several ulps.
Consequently, it remains unclear whether floating-point neural networks retain their expressive power under realistic execution semantics and general practical activation functions.

Extending representability results to such realistic floating-point execution models is highly nontrivial.
Since floating-point addition is non-associative, changing reduction orders can alter the outputs of affine transformations and intermediate network computations.
Furthermore, prior analyses often rely on exact behaviors induced by correctly rounded activation functions.
As a result, even small implementation-dependent perturbations in activation outputs may fundamentally affect the representability of floating-point neural networks.
It therefore remains unclear whether the expressive power established under idealized execution semantics persists under general practical activation functions, arbitrary reduction orders, and inexact activation implementations.
Addressing this question requires new techniques that do not rely on piecewise-linearity, fixed summation orders, or exact activation behaviors induced by correct rounding.

\subsection{Contribution}

In this work, we study when floating-point neural networks are \emph{universal representators}, i.e., whether they can represent every function between prescribed floating-point domains exactly, under generalized floating-point execution semantics.
We consider floating-point neural networks with arbitrary reduction orders and inexact activation implementations with bounded ulp errors.
Our framework substantially generalizes previous floating-point representability results by extending them beyond piecewise-linear activation functions, fixed left-to-right reduction orders, and correctly rounded activation implementations.
In particular, we establish universal representability results for a broad class of practical activation functions under significantly more realistic floating-point execution models.

Our first contribution is a general representability framework for floating-point neural networks using general activation functions.
In particular, we show that to represent all functions from a floating-point domain $\domain$ to floating-point outputs, the network must be able to distinguish every pair of distinct inputs in the first layer.
That is, for any $x,x'\in\domain$ with $x\neq x'$, there should exist a floating-point affine transformation $\phi$ such that $\sigma(\phi(x))\neq\sigma(\phi(x'))$, where $\sigma$ denotes the implemented activation function.
This gives a necessary condition for universal representability and provides a general framework for analyzing when floating-point neural networks fail to represent arbitrary floating-point functions.
As an application, we show that the correctly rounded cosine activation is not a universal representator under various floating-point formats.

Our second contribution is to prove that a suitable form of distinguishability is also sufficient for universal representability.
Under mild output conditions on the activation implementation, we construct floating-point neural networks that represent arbitrary functions on the prescribed floating-point domain.
The construction works under arbitrary reduction orders, rather than relying on a fixed left-to-right summation convention.
This shows that the universal representation phenomenon is not an artifact of a particular reduction order.
Based on this, we also provide an easily verifiable sufficient condition for correctly rounded real activation functions.

Our third contribution is to extend the theory beyond correctly rounded activation functions.
We present easily verifiable conditions under which inexact activation implementations with bounded ulp (unit in the last place) errors remain distinguishable.
As a consequence, our results apply to practical activation functions whose implementations may deviate from the correctly rounded values, including implementations of activations such as $\Sigmoid$, $\tanh$, $\relu$, $\elu$, $\mathrm{SeLU}$, $\GeLU$, $\Swish$, $\Mish$, and $\sin$.
Overall, our results establish expressive power guarantees for floating-point neural networks under significantly more realistic floating-point execution semantics than previously known. %

This manuscript is an extended version of our previous conference paper \cite{hwang2025floating}.
The conference version studied the expressive power of floating-point neural networks under fixed left-to-right floating-point summation orders and correctly rounded activation implementations.
The present manuscript substantially generalizes this setting to a more realistic scenario.
In particular, we allow arbitrary reduction orders in floating-point summations and activation implementations with bounded ulp errors, rather than assuming fixed reduction orders and correct rounding.
The journal version also strengthens the negative side of the theory by establishing more general necessary conditions and counterexamples for universal representability, where the previously known cosine-based counterexample appears as a special case.
Importantly, these generalized execution semantics fundamentally change the behavior of floating-point neural networks.
As a result, all representability results and proofs from the conference version are substantially rewritten and re-established under the new assumptions.

%% file: 2-preliminary.tex
\section{Preliminaries}

\subsection{Notations}
We use $\bbN$, $\bbZ$, and $\bbR$ to denote the sets of natural numbers, integers, and real numbers, respectively.
For $a,b\in \bbR$, we define $[a,b]\defeq \{x\in\bbR: a\le x\le b\}$ and $(a,b)\defeq \{x\in\bbR: a< x< b\}$;
we define $[a,b)$ and $(a,b]$ analogously.
For $\mcS\subset \bbR$, we define $[a,b]_{\mcS}:=[a,b]\cap \mcS$,
with $(a,b)_\mcS$, $[a,b)_\mcS$, and $(a,b]_\mcS$ defined analogously.
For $n\in\bbN$, we write $[n] \defeq [1,n]_{\bbN}$.
For $d\in\bbN$, a set $\mcS$, and $x\in\mcS^d$, we define $x_i$ as the $i$-th coordinate of $x$. 
In this paper, all fractional numbers with a radix point are assumed to be in binary representation: e.g., $1.101=2^0+2^{-1}+2^{-3}=13/8$.
Additionally, we define the ceiling function $\lceil x \rceil_{\bbZ} $ and the floor function $\lfloor x \rfloor_{\bbZ} $ as follows:
\begin{align*}
    \lceil x \rceil_{\bbZ} &\defeq \min \{ m \in \bbZ : m \ge x   \},  \\ 
    \lfloor x \rfloor_{\bbZ} &\defeq \max \{ m \in \bbZ : m \le x   \}. 
\end{align*}
\subsection{Floating-Point Arithmetic}\label{sec:float}

\paragraph{Floating-point numbers.}
For $\mbit, \ebit\in\bbN$,
we define $\fpq_{\mbit,\ebit}$ as the set of \emph{finite} floating-point numbers:
\begin{align}
    \fpq_{\mbit, \ebit}\defeq \big\{&s\times (1.m_1\cdots m_{\mbit}) \times 2^{e}: s\in\{-1,1\}, \nonumber
    \\[-2pt]
    &m_1,\dots,m_{\mbit}\in\{0,1\},e\in[\emin,\emax]_\bbZ \big\} \nonumber
    \\
    {} \cup \big\{&s\times (0.m_1\cdots m_{\mbit}) \times 2^{\emin}: \nonumber
    \\[-2pt]
    & s\in\{-1,1\},m_1,\dots,m_{\mbit}\in\{0,1\}\big\}, \label{eq:def:float}
\end{align}
where $\emin$ and $\emax$ are defined as $\emin\defeq -2^{\ebit-1}+2$ and $\emax\defeq 2^{\ebit-1}-1$. 
Here, $s$, $m_1\dots m_p$, and $e$ are called the \emph{sign}, \emph{mantissa}, and \emph{exponent} of a floating-point number, respectively. Note that $p+q+1$ bits suffice to represent all numbers in $\fpq_{p,q}$: $p$ bits and $q$ bits for representing the mantissa and exponent, respectively, and one additional bit for the sign.
For simplicity, we omit the subscript and write $\fpq$ for $\fpq_{\mbit, \ebit}$ when $p$ and $q$ are clear from the context. 

We use $\infty$ and $-\infty$ to denote positive and negative \emph{infinities}, and assume the usual order: $-\infty<x< \infty$ for any  $x\in \bbR$.
We use $\nan$ to denote \emph{not-a-number},
which can be produced, e.g., when $\infty$ is added to $-\infty$ under floating-point arithmetic. 
{We assume any operation including $\nan$ produces $\nan$; this does hold for floating-point addition, subtraction, and multiplication.}
We use $\overline{\fpq}$ to denote the set of \emph{all} floating-point numbers (or floats) $\overline{\fpq}\defeq \fpq\cup \{-\infty, \infty, \nan\}$. We note that $\overline{\fpq}_{\mbit,\ebit}$ can also be represented using $\mbit+\ebit+1$ bits since we are not using the whole $2^\ebit$ representations for the exponent of floats in $\fpq_{\mbit,\ebit}$. %
For $x\in \fpq$, $x^+$ and $x^-$ denote the smallest and largest floats 
greater than and less than $x$, respectively.
For $x\in\fpq$, we use $\expo{x}\defeq\max\{\emin,\lfloor\log_2 |x|\rfloor\}$. 
The smallest and largest finite positive floats are denoted by $\fmin \defeq 2^{\emin-\mbit}$ and $\fmax \defeq \lrp{2-2^{-\mbit}}\times 2^{\emax}$, respectively.

The IEEE-754 standard \cite{ieee754} defines $(\mbit,\ebit)$ for widely used floating-point formats:
e.g.,
$(10,5)$ for the 16-bit half precision (float16),
$ (23,8)$ for the 32-bit single precision (float32), and $(52,11)$ for the 64-bit double precision (float64).
In this paper, we assume that $(p,q)$ satisfies 
        $2 \leq \mbit \le 2^{\ebit-1} -3.$
Other popular floating-point formats also satisfy this condition, e.g., bfloat16 with $(7,8)$ \cite{bfloat,abadi2016tensorflowbfloat} and 8-bit E5M2 and E4M3 with $(2,5)$ and $(3,4)$, respectively \cite{micikevicius2022fp8}.

\paragraph{Floating-point operations.}
The {rounding} function $\round{\cdot}_{\fpq}:\bbR\cup\{-\infty,\infty,\nan\}\rightarrow \overline{\fpq}$ is defined as 
\begin{equation*}
    \round{x}_{\fpq} \defeq \begin{cases}
        \argmin_{y\in \fpq} |x-y| &\text{ if } |x|<\fmax+2^{\emax-p-1}\!,
        \\ \infty &\text{ if } x \ge \fmax+2^{\emax-p-1}\!,
        \\ -\infty &\text{ if } x \le -\fmax-2^{\emax-p-1}\!,
        \\ \nan &\text{ if }x=\nan.
    \end{cases}
\end{equation*}
There can be two floats equidistant from a real number. In such a case, we break the tie using the tie-to-even rule: $\round{x}_{\fpq}$ is defined by the (unique) float whose last mantissa bit $m_{\mbit}$ (see \cref{eq:def:float}) is zero.
If $\fpq$ is clear from the text, we omit the subscript $\fpq$ in $\round{\cdot}_\fpq$.

For $\rho:\bbR\rightarrow \bbR$, 
we define the \emph{correctly rounded} function $\round{\rho}:\efpq\rightarrow \efpq$ of $\rho$ as follows:
    \begin{equation*}
        \round{\rho}(x) \defeq \begin{cases}
            \round{\rho(x)} &\text{ if } x\in \fpq,
            \\
            \round{l}  &\text{ if } x=-\infty \land \exists \lim_{x\rightarrow -\infty}{\rho(x)}, 
             \\  \round{r}  &\text{ if } x=\infty \land \exists \lim_{x\rightarrow\infty}{\rho(x)},
             \\ \nan &\text{ otherwise},
        \end{cases} 
        \end{equation*}
where $l=\lim_{x\rightarrow-\infty}{\rho(x)}$ and $r=\lim_{x\rightarrow\infty}{\rho(x)}$.
Here, the existence of $l,r$ includes the case $l,r \in \{-\infty, \infty\}$. We define $\mathrm{ulp}(x) \defeq 2^{\expo{x}-\mbit}$.

For $x,y\in \overline{\fpq}$, we define the floating-point operations $\oplus, \ominus$, and $\otimes$ as $x\oplus y\defeq \round{x+y}$, $x\ominus y\defeq \round{x-y}$, and $x\otimes y\defeq \round{x\times y}$.
Note that the addition and multiplication are not associative: e.g., $(x\oplus y)\oplus z\neq x\oplus (y\oplus z)$ in general.
Therefore, we must be very careful about the ordering of the operations.

For $\domain\subset\fpq^d$, we use $\indcc{\domain}: \fpq^{d}\rightarrow \fpq$ to denote the indicator function of $\domain$:
$\indc{\domain}{x}$ is one if $x\in \domain$ and zero otherwise.
If $\domain = \{x_0\}$ is a singleton set, we use $\indcc{x_0}$ to denote $\indcc{\domain}$.

\subsection{Reduction Order in Floating-Point Summations}
For $n\ge2$, let $\mcO_n$ denote the set of all possible reduction orders for summing $n$ floating-point numbers.
For $n=2$, there is only one possible reduction order:
\begin{align*}
\mcO_2 = \{(1,2)\}.
\end{align*}

For $n\ge3$, each element
$o=((i_1,j_1),\dots,(i_{n-1},j_{n-1}))\in\mcO_n$
specifies a sequence of pairwise floating-point reductions.
At the first step, the $i_1$-th and $j_1$-th entries are reduced to their floating-point sum.
The remaining pairs
$o'=((i_2,j_2),\dots,(i_{n-1},j_{n-1}))$
then form an element of $\mcO_{n-1}$ for the resulting $(n-1)$-tuple.

More precisely, for $1\le i<j\le n$ and $\bfx=(x_1,\dots,x_n)$, define
\begin{align*}
\widehat{\bfx}_{i,j}
=
(x_i\oplus x_j,x_1,\dots,x_{i-1},x_{i+1},\dots,x_{j-1},x_{j+1},\dots,x_n),
\end{align*}
where the reduced value $x_i\oplus x_j$ is placed at the first coordinate.
Then, for
$o=((i_1,j_1),\dots,(i_{n-1},j_{n-1}))\in\mcO_n$,
we recursively define
\begin{align*}
\Sigma(\bfx;o)
\defeq
\Sigma(\widehat{\bfx}_{i_1,j_1};o'),
\end{align*}
where
$o'=((i_2,j_2),\dots,(i_{n-1},j_{n-1}))\in\mcO_{n-1}$,
with the base case
\begin{align*}
\Sigma((x_1,x_2);(1,2))
=
x_1\oplus x_2.
\end{align*}
For notational convenience, we also define $\mcO_1\defeq\{(1)\}$ and $\Sigma(x;(1))\defeq x$.

For example,
\begin{align*}
\mcO_3
=
\{
((1,2),(1,2)),
((1,3),(1,2)),
((2,3),(1,2))
\},
\end{align*}
and the corresponding summations are
\begin{align*}
\Sigma(\bfx_3;((1,2),(1,2)))
&=
(x_1\oplus x_2)\oplus x_3, \\
\Sigma(\bfx_3;((1,3),(1,2)))
&=
(x_1\oplus x_3)\oplus x_2, \\
\Sigma(\bfx_3;((2,3),(1,2)))
&=
(x_2\oplus x_3)\oplus x_1.
\end{align*}
Likewise, for $\bfx_4=(x_1,x_2,x_3,x_4)$,
\begin{align*}
\Sigma(\bfx_4;((1,2),(1,2),(1,2)))
&=
((x_1\oplus x_2)\oplus x_3)\oplus x_4, \\
\Sigma(\bfx_4;((1,2),(2,3),(1,2)))
&=
(x_1\oplus x_2)\oplus(x_3\oplus x_4).
\end{align*}

A \emph{reduction order} is a set
\begin{align*}
\mathcal{S}=\{o_1,o_2,o_3,\dots\},
\end{align*}
where $o_n\in\mcO_n$ for every $n\in\bbN$. %
The sequence $\mathcal{S}$ specifies how sums of floating-point numbers are reduced in the system.
For $\bfx \in \fpq^n$, $\mcS = \{o_1,o_2, o_3 , \dots \} $, we define $\Sigma(\bfx_n , \mcS)$ as
\begin{align*}
    \Sigma( \bfx , \mcS) \defeq  \Sigma( \bfx , o_{n}), \quad o_n \in \mcS 
\end{align*}

We note that left-to-right addition, i.e., sequentially adding numbers from left to right, can be represented using the reduction order.
Let
\begin{align*}
\bar{o}_n
=
\underbrace{\left( (1,2), (1,2), \dots, (1,2) \right)}_{(n-1)\text{ times}}.
\end{align*}
Then, the left-to-right addition of $n$ floating-point numbers,
\[
(((x_1 \oplus x_2)\oplus x_3)\oplus \cdots \oplus x_{n-1})\oplus x_n,
\]
can be expressed as
\begin{align*}
\Sigma(\bfx_n;\bar{o}_n)
=
(((x_1 \oplus x_2)\oplus x_3)\oplus \cdots \oplus x_{n-1})\oplus x_n .
\end{align*}

\subsection{Floating-Point Neural Network under Reduction Order}
Let $\mathcal{S}=\{o_1,o_2,o_3,\dots\}$ be a reduction order.
For $d_{\mathrm{in}},d_{\mathrm{out}}\in\bbN$, let
$w_i=(w_{i,1},\dots,w_{i,d_{\mathrm{in}}})\in\fpq^{d_{\mathrm{in}}}$
and $b_i\in\fpq$ for all $i\in[d_{\mathrm{out}}]$.
Given
\begin{align*}
I=(w_1,\dots,w_{d_{\mathrm{out}}},b_1,\dots,b_{d_{\mathrm{out}}}),
\end{align*}
we define the \emph{floating-point affine transformation under $\mcS$} by
\begin{align*}
&\mathrm{aff}_{(I,\mcS)}(x_1,\dots,x_{d_{\mathrm{in}}})_i \\
&=
\Sigma(
(w_{i,1}\otimes x_1,\dots,w_{i,d_{\mathrm{in}}}\otimes x_{d_{\mathrm{in}}});
o_{d_{\mathrm{in}}}
)
\oplus b_i
\end{align*}
for every $i\in[d_{\mathrm{out}}]$, where the bias parameter $b_i$ is added at last.
Namely, each coordinate is computed by floating-point multiplications followed by a floating-point reduction according to $o_{d_{\mathrm{in}}}$ and a floating-point bias addition.

For a floating-point activation function $\sigma:\efpq\to\efpq$, we slightly abuse notation and let $\sigma$ also denote the corresponding coordinate-wise operation on vectors:
\begin{align*}
\sigma(x_1,\dots,x_d)
=
(\sigma(x_1),\dots,\sigma(x_d)).
\end{align*}

We now define floating-point neural networks under the reduction order $\mcS$.
Let $L\in\bbN$ and let
\begin{align*}
\mathrm{aff}_{(I_\ell,\mcS)}:\efpq^{d_{\ell-1}}\to\efpq^{d_\ell}
\end{align*}
be floating-point affine transformations under $\mcS$ for all $\ell\in[L]$.
Given $\mcI=(I_1,\dots,I_L)$, we define
\begin{align}
\mcN_{(\mcI,\mcS)}
=
\mathrm{aff}_{(I_L,\mcS)}
\circ
\sigma
\circ
\dots
\circ
\mathrm{aff}_{(I_2,\mcS)}
\circ
\sigma
\circ
\mathrm{aff}_{(I_1,\mcS)}.\label{eq:nn}
\end{align}
We call $\mcN_{(\mcI,\mcS)}$ a (floating-point) \emph{$\sigma$ network under $\mcS$}.
The number of layers of the network is defined as the number of affine transformations; thus $\mcN_{(\mcI,\mcS)}$ has $L$ layers.

For technical purposes, we also consider neural networks ending with activation functions.
For $\mcI'=(I_1,\dots,I_{L-1})$, define
\begin{align*}
\mcM_{(\mcI',\mcS)}
=
\sigma
\circ
\mathrm{aff}_{(I_{L-1},\mcS)}
\circ
\dots
\circ
\sigma
\circ
\mathrm{aff}_{(I_1,\mcS)}.
\end{align*}
That is, $\mcM_{(\mcI',\mcS)}$ omits the final affine transformation and ends with an activation function.
The network $\mcM_{(\mcI',\mcS)}$ contains $L-1$ affine transformations and therefore has $L-1$ layers.

%% file: 3-main_results.tex
\section{Main Results}
In this section, we formally present our main results under an arbitrary reduction order $\mcS$.
We first introduce a necessary condition on activation functions for floating-point networks to represent floating-point functions (\cref{sec:necessary}).
We then introduce a sufficient condition and compare it with our necessary condition (\cref{sec:sufficient}).
We also provide
easily verifiable conditions on real activation functions that imply our sufficient condition under the correct rounding (\cref{sec:sufficient_real}), and under rounding with bounded errors (\cref{sec:sufficient_real2}).

\subsection{Necessary Condition on Activation Functions}\label{sec:necessary}

Given a reduction order $\mcS$, a floating-point activation function $\sigma:\efpq\to\efpq$, and a domain $\domain\subset\fpq^d$, we are interested in identifying whether $\sigma$ networks under $\mcS$ can represent all functions from $\domain$ to $\fpq\cup\{-\infty,\infty\}$.
A natural observation is that such universal representation is \emph{impossible} if there exist $x,x'\in\domain$ such that $\sigma(\phi(x))=\sigma(\phi(x'))$ for all floating-point affine transformations $\phi$ under $\mcS$. That is, for any $\sigma$ network, the outputs of the first layer at $x$ and $x'$ are identical. 
By the definition of neural networks (\cref{eq:nn}), this implies that the final outputs of the network at $x$ and $x'$ must also be identical for all $\sigma$ networks; and thus, $\sigma$ networks cannot represent any function $f:\domain \to\fpq\cup\{-\infty,\infty\}$ such that $f(x)\ne f(x')$. 

To formally describe the above idea, we define the \emph{distinguishability} of an input domain.
\begin{definition}[Distinguishability]\label{definition:distinguishable_bounded}
Let $\mcS$ be a reduction order, $\sigma:\efpq\rightarrow\efpq$, $d\in\bbN$, $\domain\subset\fpq^d$, and  $\mathcal{Y}\subset \efpq$. We say that ``$\domain$ is $\sigma$-distinguishable with range $\mathcal{Y}$ under $\mcS$'' if for any $x,x'\in \domain$ with $x\ne x'$, there exists a floating-point affine transformation $\phi$ under $\mcS$ such that
    \begin{equation}
    \label{eq:distinguishable}
    \sigma(\phi(x)) \neq\sigma(\phi(x')) 
    ~~\text{and}~~
    \sigma(\phi(\domain))\subset\mathcal{Y}.
    \end{equation}
\end{definition}

To represent all functions from $\domain\subset\fpq^d$ to $\fpq\cup\{-\infty,\infty\}$, one can observe that $\domain$ should be $\sigma$-distinguishable with range $\fpq\cup\{-\infty, \infty\}$ under $\mcS$ as stated in the following lemma. See \cref{sec:pflem:necessary} %
for the formal proof. %

\begin{lemma}\label[lemma]{lem:necessary}
Let $\mcS$ be a reduction order, $d\in\bbN$, $\domain\subset \fpq^d$, and $\sigma:\efpq\rightarrow\efpq$. If $\domain$ is not $\sigma$-distinguishable with range $\fpq\cup\{-\infty,\infty\}$ under $\mcS$, then there exists $f:\domain\to\fpq\cup\{-\infty,\infty\}$ such that  $f \neq g$ on $\domain$ for all $\sigma$ networks $g$ under $\mcS$.
\end{lemma}
We note that if an one-dimensional subset $\domain\subset \fpq$ is $\sigma$-distinguishable with some range, then for any $d\in \bbN$, $\domain^d$ is also $\sigma$-distinguishable with the same range.

Based on \cref{lem:necessary}, in the following lemma, we derive an easily verifiable necessary condition for real differentiable activation functions $\rho$ under which $\round{\rho}$ networks cannot universally represent. The proof of \cref{lem:necc_example} is in \cref{sec:pflem:necc_example}.
\begin{lemma}\label[lemma]{lem:necc_example}
Let $\mcS$ be a reduction order and $\rho: \bbR \to \bbR$ such that $ |\rho(0)| \ge (1+2^{-\mbit})\cdot 2^{-\mbit} $ and  $ |\rho'(x) | \le   2^{-\mbit}$ for $ |x | \le 2^{-\mbit-1}$. Then, $f:[-2^{ \mbit + 3 }, 2^{ \mbit + 3}]_{\fpq}\to\fpq$ with $f(0)\ne f(\omega)$ cannot be represented by a $\round{\sigma}$ network under $\mcS$.
\end{lemma}

Using \cref{lem:necc_example}, we show in \cref{cor:necessary} that networks using one of the  $\round{\cos}$, $\round{1+x^2}$, $\cosh$, and $e^{-x^2}$ activation functions cannot represent all functions from $\left[-2^{\mbit+3}, 2^{\mbit+3}\right]_{\fpq}$ to $\fpq\cup\{-\infty,\infty\}$ (e.g., $\mbit=10$ for float16 and $\mbit=23$ for float32).
The proof of \cref{cor:necessary} is in \cref{sec:pfcor:necessary}.

\begin{corollary}\label[corollary]{cor:necessary}
Let $\mcS$ be a reduction order and $\rho \in \{ \cos ,  1 + x^2 , \cosh(x) ,  e^{-x^2/2} \} $,  Then $f:[-2^{ \mbit + 3 }, 2^{ \mbit + 3}]_{\fpq}\to\fpq$ with $f(0)\ne f(\omega)$ cannot be represented by a $\round{\rho}$ network under $\mcS$.
\end{corollary}

However, for most practical activation functions, floating-point neural networks can represent all functions over a wide domain $(-2^{\emax-2}, 2^{\emax- 2})_{\fpq}$. We will see this in the next subsection.

\subsection{Sufficient Condition on Activation Functions}\label{sec:sufficient}
Although the distinguishability of $\domain\subset\fpq^d$ with range $\fpq\cup\{-\infty,\infty\}$ under $\mcS$ is necessary to represent all floating-point functions from $\domain$ to $\fpq\cup\{-\infty,\infty\}$, the distinguishability is also sufficient under mild assumptions on its range and the activation function  (\cref{thm:main}).

We first introduce the following condition to describe our sufficient condition. %
\begin{condition}\label{condition:proper_exponent}
For an activation function $\sigma:\efpq\to\efpq$,
there exist $C_0, C_1, C_2\in \fpq$ such that $|C_i| , |C_i -C_j|\le 2^{\emax}$ for all $0\le i,j\le 2$, and
\begin{align*}
\sigma(C_0)=0,
&& \!\! 
2^{\emin}\le |\sigma(C_1)| < \frac54,
&& \!\!
|\sigma(C_2)| > (2^{-\mbit-2})^+.
\end{align*}
\end{condition}
\cref{condition:proper_exponent} requires the existence of three points $C_0,C_1,C_2$ for an activation function $\sigma:\efpq\to\efpq$.
This condition can easily be satisfied for %
rounded versions of popular activation functions. 
For example, $\sigma(C_0)=0$ can be satisfied for $C_0=0$ (e.g., $\relu$, $\GeLU$, $\sin$, $\tanh$) or for $C_0$ of large magnitude such as $-2^{\emax}$ (e.g., $\Sigmoid$).
Furthermore, the conditions on $C_1,C_2$ can easily be  satisfied by choosing some $C_1=C_2$ so that $\sigma(C_1=C_2)\in((2^{-p-2})^+,5/4)$.

Under \cref{condition:proper_exponent}, we show that the distinguishability of $\domain$ with range $[-2^{\emax},2^{\emax}]$ suffices for representing all functions from $\domain$ to $(\fpq\cup\{-\infty,\infty\})^{d_\text{\rm out}}$.
The proof of \cref{thm:main} is in \cref{sec:pflem:thm_main}.
\begin{theorem}\label{thm:main}
Let $\mcS$ be a reduction order, $\sigma:\efpq\to\efpq$, $d_\text{\rm in},d_\text{\rm out}\in\bbN$, $\domain\subset \fpq^{d_\text{\rm in}}$, and $f:\domain \to\lrp{\fpq\cup\{-\infty,\infty\}}^{d_\text{\rm out}}$.
Suppose that $\sigma$ satisfies \cref{condition:proper_exponent} and $\domain$ is $\sigma$-distinguishable with range $\left[-2^{\emax }, 2^{\emax }\right]_{\fpq}$ under $\mcS$.
Then, there exists a four-layer $\sigma$ network $g$ under $\mcS$ such that $f=g$ on $\domain$.
\end{theorem}

There are two notable differences between our necessary condition (\cref{lem:necessary}) and sufficient condition (\cref{thm:main}): \cref{thm:main} requires \cref{condition:proper_exponent}, and considers a smaller range $\left[-2^{\emax }, 2^{\emax }\right]_{\fpq}$ for distinguishability. 
First, we use \cref{condition:proper_exponent} for networks to generate all possible values in $\fpq\cup\{-\infty,\infty\}$ (see \cref{lem:last_layer_sums}). If an activation function $\sigma$ can only output too large values (e.g., $[2^{\emax}, \infty]$) or too small values (e.g., in $[-\omega,\omega]$), then a $\sigma$ network may not be able to generate some values in $\fpq\cup\{-\infty,\infty\}$.
Second, the smaller range is due to technical reasons in our proof, which is used to avoid overflow during the evaluation of networks. If a network can generate values with large absolute values (e.g., close to $\Omega$) while distinguishing inputs, then multiplying/adding constants to those values may incur overflow and the network may output NaN.
However, if $\sigma$ has a well-bounded range (i.e., $\sigma(\fpq\cup\{-\infty,\infty\})\subset[2^{-\emax}, \allowbreak 2^{\emax}]$) as in the correctly rounded versions of $\Sigmoid$ and $\tanh$, then this range condition is automatically satisfied.

\subsection{Sufficient Conditions for Distinguishability using Correctly Rounded Activation Functions}\label{sec:sufficient_real}
Based on our sufficient condition,
in this section, we provide easily verifiable conditions
on floating-point activation functions
(\cref{lemma:distinguishing_points_to_distinguishable,lemma:sigmoidal_distinguish})
and real activation functions
(\cref{lemma:real_function_distinguishing})
that imply distinguishability, when the activation functions are correctly rounded.
Using these conditions, we show that networks using correctly rounded versions of popular activation functions can  represent floating-point functions.
(\cref{cor:real_function_distinguishing,cor:sigmoidal_distinguish}).
Specifically, we focus on $\Sigmoid$, $\tanh$, ${\mathrm{Identity}}$, ${\relu}$, ${\elu}$, ${\mathrm{SeLU}}$, ${\GeLU}$, ${\Swish}$, ${\Mish}$, and ${\sin}$, whose definitions are presented in \cref{subsec:activation}.

To describe our conditions, we first define the \emph{separating points} of a floating-point activation function $\sigma$.

\begin{definition}[Separating Point]\label{def:separating_point}
For $\sigma:\efpq\to\efpq$, we call $\eta\in\fpq$ a ``separating point of $\sigma$'' if 
\begin{align*}
    \sigma(\eta^-) \notin \{\sigma(\eta), \sigma(\eta^+)\}
    \quad\text{or}\quad
    \sigma(\eta^+) \notin \{\sigma(\eta), \sigma(\eta^-)\}.
\end{align*}
Here, $\eta^-$ (or $\eta^+$) denotes the largest (or smallest) float that is smaller (or larger) than $\eta$ (see  \cref{sec:float}). 
\end{definition}
We design our sufficient conditions using separating points. Specifically, for each distinct pair $(x,x')$ of inputs in a domain $\domain$, we aim to find a floating-point affine transformation $\phi_{x,x'}$ and a separating point $\eta_{x,x'}$ of $\sigma$
such that $\phi_{x,x'}(x)=\eta_{x,x'}^-$ (or $\eta_{x,x'}^+$) and $\phi_{x,x'}(x')\in\{\eta_{x,x'},\eta_{x,x'}^+\}$ (or $\{\eta_{x,x'}^-,\eta_{x,x'}\})$.
If we can find such an affine transformation and a separating point for all distinct pairs in the domain, then the domain is distinguishable
since $\sigma(\phi_{x,x'}(x))\ne\sigma(\phi_{x,x'}(x'))$ (see \cref{eq:distinguishable}).

Based on the above idea, we propose our sufficient condition for distinguishability in the following lemma. The proof of \cref{lemma:distinguishing_points_to_distinguishable} is in \cref{sec:pflem:distinguishing_points_to_distinguishable}.

\begin{lemma}\label[lemma]{lemma:distinguishing_points_to_distinguishable}
Let $\mcS$ be a reduction order, $\sigma:\efpq\to\efpq$, $n\in\bbN$, and $\eta_1,\dots,\eta_n\in\fpq$ be separating points of $\sigma$ with $|\eta_1|\le \dots \le |\eta_n|$.
Suppose that integers $e_1, e_2 \in [\emin+1, \emax]$ satisfy
\begin{equation*}
    [\emin, e_2]_{\bbZ}\subset \bigcup_{i=1}^n \left[\expo{\eta_i}-e_1, \expo{\eta_i}+\emax-2\right]_{\bbZ}.
\end{equation*}
Then, $(-2^{e_2+1}, 2^{e_2+1})_{\fpq}$ is $\sigma$-distinguishable with range
\begin{equation*}
    \mcR_{e_1,e_2}\defeq\sigma\big([ -( 2^{e_1+e_2+1} \oplus |\eta_n|^+),  2^{e_1+e_2+1} \oplus |\eta_n|^+ ]_{\fpq}\big)
\end{equation*}
under $\mcS$.
\end{lemma}

In \cref{lemma:distinguishing_points_to_distinguishable}, one can observe that having two separating points (one with small magnitude and one with moderate-to-large magnitude) suffices to distinguish a large domain.
In particular,
if $\sigma$ has two separating points $\eta_1$, $\eta_2$ with 
$${|\eta_1|< 2^{\emin+1}~~\text{and}~~4\le|\eta_2|<2^{\emax-p-1}},$$ 
then $(-2^{e_2+1}, 2^{e_2+1})_{\fpq}$ is $\sigma$-distinguishable with range $\mcR_{0,e_2}$ {for all $e_2 \in [\emin+1, \emax]_\bbZ$}.
By choosing the largest $e_2^*$ such that $\mcR_{0,e_2^*}\subset[-2^{\emax},2^{\emax}]$ and by using 
\cref{thm:main}, we can show that $\sigma$ networks can represent all functions from $(-2^{e_2^*+1}, 2^{e_2^*+1})_{\fpq}$ to $\fpq\cup\{-\infty,\infty\}$.
For example, $e^*_2=\emax-1$ when $\sigma = \round{\relu}$.

We note that having a separating point with a small absolute value (e.g., $\approx 2^{\emin}$) is critical for distinguishing a large domain using \cref{lemma:distinguishing_points_to_distinguishable}, while avoiding overflow.
This is because, to distinguish two small numbers (e.g., $0$ and $\omega$), we find $w,b\in\fpq$ and a separating point $\eta \in \fpq$ such that $w\otimes0\oplus b=\eta^-$ (or $\eta^+$) and $w\otimes\omega\oplus b\in\{\eta,\eta^+\}$ (or $\{\eta,\eta^-\}$).
If $\eta$ is large in magnitude, then $w$ and $b$ must also be large. As a result, overflow may occur when the domain contains large numbers.
We also note that a separating point with moderate-to-large absolute value (e.g., $\ge1$) is necessary to distinguish a large domain using \cref{lemma:distinguishing_points_to_distinguishable}.

Using \cref{lemma:distinguishing_points_to_distinguishable}, we provide a sufficient condition for the distinguishability of floating-point activation functions that have finite values at all inputs including $\pm\infty$, such as $\round{\Sigmoid}$. %
The proof of \cref{lemma:sigmoidal_distinguish} is in \cref{sec:pflem:sigmoidal_distinguish}. %
\begin{lemma}\label[lemma]{lemma:sigmoidal_distinguish}
Let $\mcS$ be a reduction order, $\sigma:\efpq\to\efpq$ such that $\sigma(\fpq\cup\{-\infty,\infty\})\subset[-2^{\emax},2^{\emax}]_{\fpq}$.
Suppose that $\sigma$ has two separating points $|\eta_1|<2$ and $|\eta_2|\ge 4$.
Then, $\fpq$ is $\sigma$-distinguishable with range $\left[-2^{\emax}, 2^{\emax}\right]_{\fpq}$ under $\mcS$.
Suppose $\sigma$ has one separating point $\eta_1 \in \fpq$ such that $1\le |\eta_1|<2$. Then $(-2^{\emax}, 2^{\emax})_\fpq$ is $\sigma$-distinguishable with range $[-2^{\emax}, 2^{\emax}]_\fpq$ under $\mcS$.

\end{lemma}
\cref{lemma:sigmoidal_distinguish} states that if
$|\sigma(\fpq\cup\{\infty,-\infty\})|$ is %
bounded by $2^{\emax}$
and there exist  two separating points of moderate size, then $\fpq$ is $\sigma$-distinguishable with range $[-2^{\emax},2^{\emax}]_{\fpq}$. If $\sigma$ has one separating point of moderate size, then $(-2^{\emax}, 2^{\emax})_\fpq$  is $\sigma$-distinguishable with range $[-2^{\emax},2^{\emax}]_{\fpq}$.

By \cref{thm:main}, this implies that $\sigma$ networks can represent {all functions from $\fpq^d$ to $\fpq\cup\{-\infty,\infty\}$},
leading to the following corollary for $\round{\Sigmoid}$ and $\round{\tanh}$.
The proof of \cref{cor:sigmoidal_distinguish} is presented in \cref{sec:pflem:sigmoidal_distinguish}.
\begin{corollary}\label[corollary]{cor:sigmoidal_distinguish}
Let $\mcS$ be a reduction order and $\sigma$ be one of $\round{\Sigmoid}$ and $\round{\tanh}$.
Then, for any $d \in \bbN$, $\sigma$ networks under $\mcS$ can represent all functions from $\domain$ to $\fpq \cup \{-\infty, \infty\}$ where 
\begin{align*}
    \domain = \begin{cases}
    \fpq^d \; &\text{if} \; \sigma=\round{\Sigmoid}~~\text{or}~~\mbit \ge 9,  \\ 
(-2^{\emax}, 2^{\emax})_{\fpq}^d \; &\text{if} \; 2 \le \mbit \le 8~~\text{and}~~ \sigma = \round{\tanh}.   \end{cases}
\end{align*}
\end{corollary}

While \cref{lemma:sigmoidal_distinguish} implies \cref{cor:sigmoidal_distinguish}, its condition is for floating-point activation function, that may not be easy verified when the activation function is a rounded version of a real function.
We next consider for a more realistic scenario: $\sigma$ is the correctly rounded version of a real activation function $\rho$.
For this case, we first introduce a sufficient condition for having a separating point. The proof of \cref{lemma:single_distinguishing_point} is in \cref{sec:pflem:single_distinguishing_point}.

\begin{lemma}\label[lemma]{lemma:single_distinguishing_point}
    Let $\rho:\bbR\to\bbR$, $a,b\in\fpq$, and $e\in [\emin, \allowbreak \emax]_{\bbZ}$ with $|\rho(a)|, |\rho(b)| \le 2^{e+1}(1-2^{-\mbit})$.    Suppose that there exists $L>0$ such that $L (b-a) \ge 2^{e-\mbit}$ and either
\begin{itemize}[noitemsep,nosep]
\item $\rho(a)\ge 0$ and $\resig'(x) \ge L$ for all $x\in [a,b]$, or 
\item $\rho(a)\le 0$ and $\resig'(x) \le -L$ for all $x\in [a,b]$.
\end{itemize}       
Then, there exists a separating point $\eta\in [a,b)_{\fpq}$ of $\round{\rho}$.
\end{lemma}
\cref{lemma:single_distinguishing_point} states that if $\rho$ is sufficiently increasing or decreasing on a long enough interval, then that interval contains a separating point of $\round{\rho}$.
Using this lemma, we now present a sufficient condition on a real activation function $\rho$ that guarantees the distinguishability of its correctly rounded version $\sigma=\round{\rho}$.
The proof of \cref{lemma:real_function_distinguishing} is in  \cref{sec:pflem:real_function_distinguishing}.
\begin{figure}
    \centering
\includegraphics[width=\linewidth]{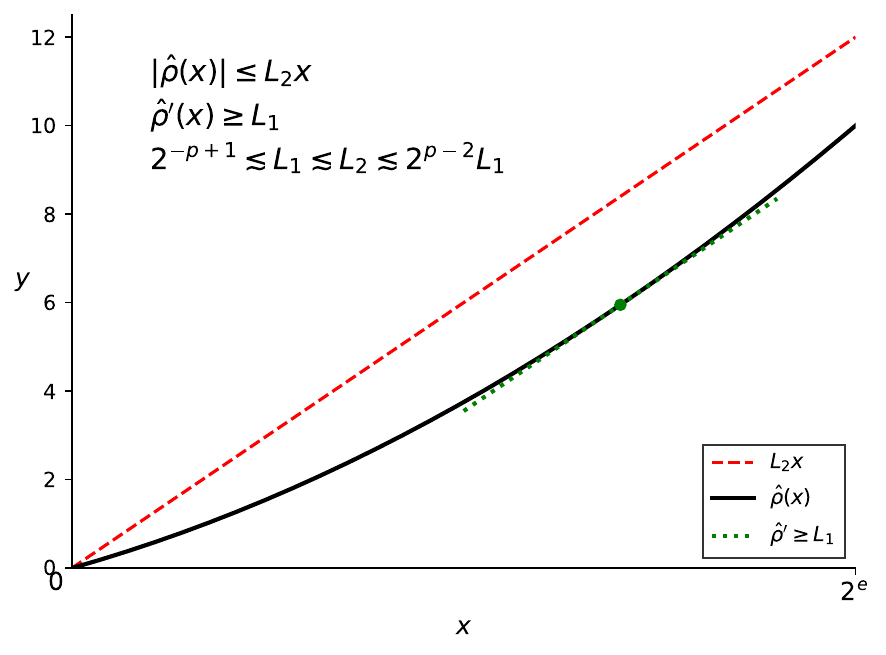}
    \caption{Visualization of the conditions in \cref{lemma:real_function_distinguishing}.}
\label{fig:real_function_distinguishing}
\end{figure}

\begin{lemma}\label[lemma]{lemma:real_function_distinguishing}
Let $\mcS$ be a reduction order, $\resig,\zeta:\bbR\to\bbR$ with $\zeta(x)=-x$. %
Suppose that there exist $\hat\rho\in\{\rho,\zeta\circ\rho,\rho\circ\zeta,\zeta\circ\rho\circ\zeta\}$, $L_1, L_2>0$, and $e\in\bbZ$
such that the following hold:
\begin{itemize}[noitemsep,nosep]
    \item $\hat\resig'(x) \geq L_1$ for all $x \in (0, 2^e)$. 
    \item $\left|\hat\resig(x)\right| \le L_2 x$ for all $x\in(-2^e,2^{e})$.
    \item %
   $-\mbit \leq l_1 \leq l_2 \leq l_1 +\mbit$,
   where $l_1= \lfloor \log_2 (L_1/2) \rfloor_{\bbZ}$ and $l_2= \lfloor \log_2 (2L_2) \rfloor_{\bbZ}$.
\end{itemize}
Then, for $e'=\emax+\min\{e -2, -l_2\}$, $(-2^{e'}, 2^{e'})_{\fpq}$ is $\round{\rho}$-distinguishable with range $\left[-2^{\emax}, \allowbreak 2^{\emax}\right]_{\fpq}$ under $\mcS$.
\end{lemma}
\cref{lemma:real_function_distinguishing} roughly states that
for some moderate-size $L_1, L_2 > 0$,
if a real activation function $\rho$ is bounded between $L_1x$ and $L_2x$, and
$\rho'$ is lower bounded by $L_1$ for all inputs between $0$ and $2^e$,
then the domain $(-2^{e'},2^{e'})_{\fpq}$ is $\round{\rho}$-distinguishable with range $[-2^{\emax},2^{\emax}]_{\fpq}$, where $e' \approx \emax + e-2$.
By \cref{thm:main}, this implies that $\round{\rho}$ networks can represent all functions from $(-2^{e'},2^{e'})_{\fpq}$ to $\fpq\cup\{-\infty,\infty\}$.
Since \cref{lemma:real_function_distinguishing} covers various practical activation functions $\rho$ with $\rho(0)=0$, %
the following corollary holds for the correctly rounded version of such $\rho$.
The proof of \cref{cor:real_function_distinguishing} is in  \cref{sec:pflem::real_function_distinguishing}.

\begin{corollary}\label[corollary]{cor:real_function_distinguishing}
Let $\mcS$ be a reduction order and $\sigma$ be one of $\round{\mathrm{Identity}}$, $\round{\relu}$, $\round{\elu}$, $\round{\mathrm{SeLU}}$, $\round{\GeLU}$, $\round{\Swish}$, $\round{\Mish}$, and $\round{\sin}$.
Then, for any $d\in \bbN$, $\sigma $ networks under $\mcS$ can represent all functions from $(-2^{\emax-1}, 2^{\emax- 1})_{\fpq}^d$ to $\fpq\cup\{-\infty,\infty\}$.
\end{corollary}

Surprisingly, \cref{cor:real_function_distinguishing} shows that floating-point networks using the identity activation function can represent all floating-point functions. Such an observation can be made since floating-point affine transformations are \emph{non-affine} under exact arithmetic due to the round-off errors, i.e., the networks can be non-affine. We note that under real arithmetic, networks using the identity activation function must be affine, and hence, cannot be universal approximators. 

\subsection{Sufficient Conditions for Distinguishability using Activations with Bounded Ulp Errors}\label{sec:sufficient_real2}

In practice, implementations of activation functions are often not correctly rounded.
While correctly rounded implementations return the nearest floating-point value of the target real function, practical implementations often allow small approximation errors to reduce computational cost and implementation complexity.
For example, elementary functions such as $\exp$, $\tanh$, $\sin$, and $\cos$ are commonly implemented using polynomial approximations, table-based methods, or hardware-specific instructions, which may introduce small floating-point errors.
As a result, practical floating-point implementations are often analyzed using ulp (\emph{unit in the last place}) error bounds rather than exact correct-rounding guarantees.

For $x\in\bbR$ such that $\round{x}\in\fpq$ and $2^e\le x<2^{e+1}$ for some $e\in\bbZ$ ($e=-\infty$ if $x=0$), we define
\begin{align*}
\ulp(x)\defeq2^{\max\{e,\emin\}-\mbit}.
\end{align*}
Intuitively, $\ulp(x)$ represents the spacing between nearby floating-point numbers around $x$.
Under the round-to-nearest mode, a correctly rounded implementation has at most $0.5$ ulp error.

Given $\rho:\bbR\to\bbR$, $\sigma:\efpq\to\efpq$, and $\mcZ\subset\fpq$, we say that \emph{$\sigma$ approximates $\rho$ with $K$ ulp error on $\mcZ$} if $\round{\rho}(\mcZ)\subset\fpq$ and
\begin{align*}
|\sigma(x)-\rho(x)|
\le
K\times\ulp(\rho(x))
\end{align*}
for all $x\in\mcZ$.

Practical floating-point implementations of elementary functions often provide ulp-error guarantees or empirical ulp-error bounds rather than exact correct rounding.
For single precision, \cite{gladman2025accuracy} reports exhaustive accuracy measurements over all float32 inputs for univariate functions.
Their results show that, for most elementary functions relevant to activation implementations, such as $\exp$, $\log$, $\sin$, $\cos$, and $\tanh$, the observed ulp errors are below $1$ for several widely used mathematical libraries.
Motivated by these observations, we study the expressive power of floating-point neural networks with activation functions allowing bounded ulp errors.

Here, we mainly focus on high-precision floating-point formats such as float16, bfloat16, float32, and float64.
In contrast, for very low-precision formats such as float8, correctly rounded activation functions can often be implemented efficiently using lookup tables because the number of representable values is small.

To describe our results, we present our assumption on a real activation function and its implementation.
\begin{condition}\label{cond:impl}
Consider a real number $ 0 \le K\le 10$.
For a real activation function $\rho:\bbR\to\bbR$ and its floating-point implementation $\sigma:\efpq\to\efpq$, $\sigma$ approximates $\rho$ in $K$ ulp errors on $[0,8]_{\fpq}$, $|\rho([0,8])| \le 2^{\emax} $, $|\sigma(x)| \le 2|\rho(x)|$ for $|x|
\ge 1$ and $\sigma(C_0)=0$ for some $C_0 \in \fpq$. 
\end{condition}

The assumptions in \cref{cond:impl} are motivated by practical implementations.
Referring to \cite{gladman2025accuracy}, several mathematical libraries such as GNU libc and LLVM libc report ulp errors below $1$ for many elementary functions in single precision.
Although widely used activations such as $\GeLU$ and $\Swish$ are not themselves elementary functions, they are compositions of elementary functions, and their implementations are expected to exhibit moderate ulp errors.
Accordingly, \cref{cond:impl} assumes bounded ulp error on $[0,8]$.
We also report the maximum ulp errors of several activations on $[0,8]$ in half and single precision in \cref{table:ulp}.

The condition $\sigma(C_0)=0$ for some $C_0\in\fpq$ is also natural in practice.
For example, although the real sigmoid function never attains zero, its floating-point implementation can underflow to zero for sufficiently negative inputs.
Under \cref{cond:impl}, implementations with bounded ulp error still admit separating points, which allows us to establish universal representability results for $\Sigmoid$ and $\tanh$ networks.
The proof of \cref{cor:sigmoidal_distinguish_nc} is in \cref{sec:pflem:sigmoidal_distinguish_nc}.

\begin{corollary}\label[corollary]{cor:sigmoidal_distinguish_nc}
Let $\mbit \ge 7$, $d\in\bbN$, and  $\mcS$ be a reduction order. Suppose $\rho\in \{ \Sigmoid, \tanh \}$ and $\sigma:\efpq\to\efpq$ satisfy \cref{cond:impl}.
Then, $\sigma$ networks under $\mcS$ can represent all functions from $\domain$ to $\fpq \cup \{-\infty, \infty\}$ where 
\begin{align*}
    \domain = \begin{cases}
    \fpq^d \; &\text{if} \; \sigma=\round{\Sigmoid}~~\text{or}~~\mbit \ge 15,  \\ 
(-2^{\emax}, 2^{\emax})_{\fpq}^d \; &\text{if} \; 7 \le \mbit \le 14~~\text{and}~~ \sigma = \round{\tanh}.   \end{cases}
\end{align*}
\end{corollary}

Unlike the correctly rounded setting, implementations with bounded ulp error require an additional assumption such as $\mbit\ge7$.
The main reason is to guarantee the existence of separating points.
In the correctly rounded setting, if a real activation function is monotone on an interval and sufficiently many floating-point gaps exist, then its floating-point implementation admits separating points (e.g., \cref{lemma:single_distinguishing_point}).
In contrast, bounded ulp error can collapse nearby floating-point outputs into the same floating-point value, and therefore larger output gaps are required to guarantee distinguishability.
Consequently, establishing separating points requires sufficiently many floating-point gaps, which leads to additional precision requirements such as $\mbit\ge7$ (e.g., \cref{lemma:single_distinguishing_point_impl}).

As in the correctly rounded setting, we next provide an easily verifiable sufficient condition on a real activation function and its implementation.
Using this condition, we establish universal representability results for networks using implementations of $\mathrm{Identity}$, $\relu$, $\elu$, $\mathrm{SeLU}$, $\GeLU$, $\Swish$, $\Mish$, and $\sin$ under bounded ulp error.
The proofs of \cref{lemma:single_distinguishing_point_impl,lemma:real_function_distinguishing_ncrounded} are in \cref{sec:pflem:single_distinguishing_point_impl,sec:pflem:real_function_distinguishing_ncrounded,sec:pfcor:real_function_distinguishing_ncrounded}, respectively.
\begin{lemma}\label[lemma]{lemma:single_distinguishing_point_impl}
Let $\mbit\ge7$ and suppose $ \rho  : \bbR \to \bbR$ and $\sigma:
\efpq\to\efpq$  satisfies \cref{cond:impl}. 
    Let $a,b\in [0,8]_\fpq$, and $e\in [\emin, \allowbreak \emax]_{\bbZ}$ with $ |\rho(a)|, | \rho(b) |  \leq (1-2^{-\mbit}) \cdot 2^{e+1}$.
    Suppose that $\rho$ is monotone on $[a,b]$ and  
    there exists $L>0$ such that $L (b-a) \ge 2^{e-\mbit+5}$ and either
\begin{itemize}[noitemsep,nosep]
\item $\rho(a)\ge 0$ and $\resig'(x) \ge L$ for all $x\in [a,b]$, or 
\item $\rho(a)\le 0$ and $\resig'(x) \le -L$ for all $x\in [a,b]$.
\end{itemize}       
Then, there exists a separating point $\eta\in [a,b]_{\fpq}$ of $\sigma$.
\end{lemma}

\begin{figure}
    \centering
\includegraphics[width=\linewidth]{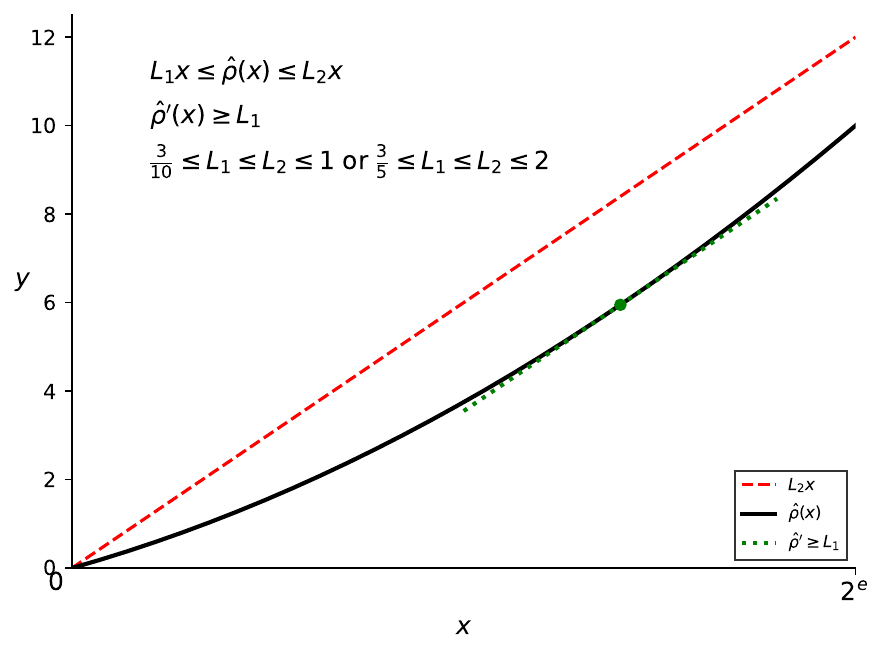}
    \caption{Visualization of the conditions in \cref{lemma:real_function_distinguishing_ncrounded}.}
\label{fig:real_function_distinguishing_nc_rounded}
\end{figure}

\begin{lemma}\label[lemma]{lemma:real_function_distinguishing_ncrounded}
Let  $\mbit \ge 7$, $\mcS$ be a reduction order, and $\zeta:\bbR\to\bbR$ with $\zeta(x)=-x$.
Suppose that $\rho:\bbR\to\bbR$ and $\sigma : \efpq \to \efpq$ satisfy \cref{cond:impl}.
Assume $ 
  \tfrac{3}{10}  \le L_1 \le L_2 \le 1$ or $
 \tfrac{3}{5} \le L_1 \le L_2 \le 2$.
Suppose that there exist $\hat\rho\in\{\rho,\zeta\circ\rho,\rho\circ\zeta,\zeta\circ\rho\circ\zeta\}$ 
such that the following hold:
\begin{itemize}[noitemsep,nosep]
    \item $\hat\resig'(x) \geq L_1$ for all $x \in (0, 2^{\emax})$,
    \item $ |\hat\resig(x) | \le L_2 x$ for all $x\in (-2^{\emax+1},2^{\emax+1})$,
\end{itemize}
Pick $k \in \bbZ$ such that $  2^{k-1}  < L_2 \le 2^k$.
Then, $(- 2^{\emax-1-k}, 2^{\emax-1-k} )_\fpq$ is $\sigma$-distinguishable with range $\left[ 2^{\emax}, \allowbreak 2^{\emax}\right]_{\fpq}$ under $\mcS$.
\end{lemma}

\begin{corollary}\label[corollary]{cor:real_function_distinguishing_ncrounded}
Let $\mbit \ge 7$ and  $\mcS$ be a reduction order. Suppose that $\rho \in \{ \mathrm{Identity} $, $\relu$, $\elu$, $\mathrm{SeLU}$, $\GeLU$, $\Swish$, $\Mish$, $\sin \}$ and $\sigma:\efpq\to\efpq$ satisfy \cref{cond:impl}.
Then, for any $d\in \bbN$, $\sigma $ networks under $\mcS$ can represent all functions from $(-2^{\emax-2}, 2^{\emax- 2})_{\fpq}^d$ to $\fpq\cup\{-\infty,\infty\}$.
\end{corollary}

%% file: 4-proofs.tex
\section{Additional notations and definitions for proofs}

We denote the left-to-right addition 
$$ \bar{o}_n = \underbrace{\left( (1,2), (1,2), \dots , (1,2) \right)}_{(n-1) \; times},$$ 
which is summed as 
\begin{align*}
    &\Sigma(\bfx_n;\bar{o}_n) = ( \dots (x_1 \oplus x_2) \oplus x_3  ) \oplus \dots \oplus x_{n-1})\oplus x_n . 
\end{align*}
We also define 
\begin{align*}
    \overline{\mcS}  = \{ \bar{o}_i  :  i=1,2,3,\dots  \} = \{\bar{o}_1, \bar{o}_2 , \bar{o}_3 , \dots  \}
\end{align*}
for the set of the left-to-right additions, where $\bar{o}_1=(1)$. 

\begin{definition}[Embedding]
Let $m>n$.
We define an embedding
\[
\iota : \fpq^n \to \fpq^m
\]
such that, for any $\bfx=(x_1,\dots,x_n)\in\fpq^n$,
\[
\iota(\bfx)=(y_1,\dots,y_m),
\]
and there exists an injective map $I:[n]\to[m]$ satisfying
\[
x_i = y_{I(i)}
\]
for all $i\in[n]$.
\end{definition}

\begin{definition}[Transferability]\label{def:transferability}
Define $\bbS_\sigma\defeq \{w\otimes \sigma(c): w,c\in \fpq~\text{with}~w\otimes \sigma(c)\in\fpq\}.$   Let $\mcS$ be a reduction order, $m\in\bbN$ and $(x_1,\dots,x_m)$, $(y_1,\dots, y_m)\in \fpq^m$. We say ``$(x_1,\dots, x_m)$ is transferable to $(y_1,\dots,y_m)$ with $(\sigma, \bfz, j,  \mcS)$'' or write 
\begin{align*}
    (x_1, \dots, x_m)  \xRightarrow{(\sigma, \bfz, j, \mcS)}  (y_1,\dots, y_m),
\end{align*}
    if there exist
    $\bfz = (z_1 , \dots z_n ) \in \bbS_\sigma^n$ and embedding $\iota_{\bfz,j} : \fpq \to \fpq^{n+1}$ for $j \in \{0\} \cup [n]$  such that          
    \begin{align*}
    \iota_{\bfz,j} (x) = (z_1 , \dots, z_j , x , z_{j+1}, \dots , z_n ), \; \Sigma( \iota_{\bfz,j} (x_i) , \mcS ) = y_i,
    \end{align*}
    for all $i\in [m]$. 
For simplicity, we often omit $\bfz, j$ and simply write
\begin{align*}
    (x_1, \dots, x_m)  \xRightarrow{(\sigma, \mcS)}  (y_1,\dots, y_m). 
\end{align*}
For left-to-right addition $\overline{\mcS}$, we denote
$$(x_1, \dots, x_m)  \xRightarrow{(\sigma, \overline{\mcS})}  (y_1,\dots, y_m)$$ as 
\begin{align*}
    (x_1, \dots, x_m)  \xRightarrow{(\sigma, \bfz , 0 ,  \overline{\mcS})}  (y_1,\dots, y_m). 
\end{align*}
 We often drop $\sigma$ and use $\bbS$ to denote $\bbS_\sigma$ if it is clear from the context.
\end{definition}

\subsection{Activation functions}\label{subsec:activation}
We present mathematical definitions of popular activation functions including $\Sigmoid$, $\relu$, $\elu$ \cite{clevert2015fastelu}, $\Swish$ \cite{ramachandran2017searchingswish}, $\Mish$ \cite{misra2019mish}, $\GeLU$ \cite{hendrycks2016gaussian}, $\mathrm{SeLU}$ \cite{klambauer2017selfselu} :
\begin{itemize}[leftmargin=0.2in]
    \item $\Sigmoid(x) \defeq \frac{1}{1+e^{-x}}$,
    \item $\relu(x) \defeq \max(0,x)$,
        \item $        \ELU(x) \defeq \begin{cases}
            x  &\text{ if } x\ge 0,
            \\ \exp(x)-1 &\text{ if } x< 0,
        \end{cases}$
        \item $\Swish(x) \defeq \frac{x}{1 + e^{-x}}$, %
        \item $\Mish(x) \defeq x\tanh(\SoftPlus(x))$, %
        \item $\GeLU(x) \defeq \tfrac{x}{2}\left(1+ \tfrac{2}{\sqrt{\pi}} \int_0^{\tfrac{x}{\sqrt{2}}  } e^{-t^2}dt  \right)$ ,
\item $\mathrm{SeLU}(x)=
\lambda
\begin{cases}
x, & x > 0,\\[4pt]
\alpha (e^x - 1), & x \le 0,
\end{cases}$, \\
\[  \alpha \approx 1.6732,
\;
\lambda \approx 1.0507.\]
\end{itemize}

\begin{table}[h]
\centering
    \begin{tabular}{cccccc@{\;\;\;}ccccc}
\toprule
\begin{tabular}{@{}c@{}}
  $\rho(x)$ 
\end{tabular}
& \begin{tabular}{@{}c@{}}
  max ulp error \\
   binary16
\end{tabular} & \begin{tabular}{@{}c@{}}
  max ulp error \\
   binary32
\end{tabular}    \\ 
\midrule
$\relu$ & 0  & 0     \\
$\elu$  & 0  & 0 \\
$\GeLU$ & 2  & 9 \\
$\text{SeLU}$ & 1 &  2\\
$\Swish$ & 2  & 2\\
$\Mish$  & 2  & 3 \\
\bottomrule
\end{tabular}
\vspace{0.2in}
\caption{%
Ulp errors in popular activations on $[0,8]$ in half and single precision.
}
\label{table:ulp}
\end{table}

\section{Technical lemmas on transferability under left-to-right addition}

\begin{lemma}\label[lemma]{lemma:nonzero_to_Zero}
Let $\sigma:\efpq\to\efpq$ and suppose that $\sigma$ satisfies \cref{condition:proper_exponent}.
Then, for any $x_1,x_2\in[-2^{\emax},2^{\emax}]_{\fpq}$ with $x_1< x_2$, there exists a constant $y\in(0,2^{\emax}]_{\fpq}$ such that %
    $$(x_1,x_2)\tranLR (0,y).$$
\end{lemma}
\begin{proof}
    Without loss of generality, assume that $x_1<x_2$.
    Use mathematical induction on $|x_1|$.
If $x_1 = 0$, there is nothing to prove.
Assume that if $|x'|<x_1$ for any $x_2'\neq x'$, then there exists $y'$ such that $(x',x'_2)\tranLR (0,y')$.

If $0<x_1<x_2$, by \cref{lem:endbit_control}, there exists $\gamma\in\bbS_\sigma$ such that $\gamma<x_1$ and $x_1 \oplus (-\gamma) <x_1$ and $x_2\oplus (-\gamma) \neq x_1\oplus (-\gamma)$.
Therefore, 
\begin{equation*}
    (x_1,x_2)\tranLR (x_1\oplus (-\gamma), x_2 \oplus (-\gamma)) \tranLR (0,y'),
\end{equation*}
by the induction hypothesis.

If $x_1<0<x_2$, by \cref{lem:endbit_control}, there exists $\gamma\in \bbS_\sigma$ such that  $x_1 \oplus \gamma > x_1$.
Therefore, 
\begin{equation*}
    (x_1,x_2)\tranLR (x_1\oplus \gamma, x_2 \oplus \gamma) \tranLR (0,y'),
\end{equation*}
by the induction hypothesis.
This completes the proof.
\end{proof}

\begin{lemma}\label[lemma]{lemma:killing_small}
Assume that $\sigma$ satisfies \cref{condition:proper_exponent}.
Consider $x_1, x_2\in \fpq$ such that $\expo{x_1} \le \expo{x_2}-2$.
Then, 
\begin{equation*}
    (x_1, x_2)\tranLR (0,x_2).   
\end{equation*}
\end{lemma}

\begin{proof}[Proof of \cref{lemma:killing_small}]
    Fix $x_2$ and use mathematical induction on the absolute value of $x_1$ to prove the statement.
    If $x_1=0$, there is nothing to prove.
    Assume that for $x'\in\fpq$ such that $|x'|<|x_1|$, the induction hypothesis is satisfied.
    By \cref{lem:endbit_control}, there exists $\gamma$ such that $\gamma\in \left(2^{\expo{x_2}-\mbit-3}, 2^{\expo{x_2}-\mbit-2}\right]_{\fpq}\cap \bbS_\sigma$.
    Then, $x_2\oplus (\pm\gamma) = x_2$.
    For $x_1>0$, $|x_1\oplus (-\gamma)| < x_1$, and for $x_1<0$, $|x_1\oplus \gamma| < x_1$.
    Therefore, there exists $x'$ such that $|x'|<|x_1|$, and 
    \begin{equation*}
        (x_1, x_2)\tranLR (x',x_2) \tranLR (0, x_2),
    \end{equation*}
    where the last relation is by the induction hypothesis.
    This completes the proof.
\end{proof}

\begin{lemma}\label[lemma]{lemma:zero_each_other}
Let $\sigma:\efpq\to\efpq$ and suppose that $\sigma$ satisfies \cref{condition:proper_exponent}.
Then, for any $x_1,x_2\in[-2^{\emax},2^{\emax}]_{\fpq}$ with $x_1x_2>0$, we have
\begin{equation*}
    (0,x_1) \tranLR (0,x_2).
\end{equation*}
\end{lemma}
\begin{proof}
    Without loss of generality, assume that $x_1<x_2$.
    Use mathematical induction on $|x_1|$.
If $x_1 = 0$, there is nothing to prove.
Assume that if $|x'|<x_1$ for any $x_2'\neq x'$, then there exists $y'$ such that $$(x',x'_2) \tranLR (0,y').$$

If $0<x_1<x_2$, by \cref{lem:endbit_control}, there exists $\gamma\in\bbS_\sigma$ such that $\gamma<x_1$ and $x_1 \oplus (-\gamma) <x_1$ and $x_2\oplus (-\gamma) \neq x_1\oplus (-\gamma)$.
Therefore, 
\begin{equation*}
    (x_1,x_2)\tranLR (x_1\oplus (-\gamma), x_2 \oplus (-\gamma)) \tranLR (0,y'),
\end{equation*}
by the induction hypothesis.

If $x_1<0<x_2$, by \cref{lem:endbit_control}, there exists $\gamma\in\bbS_\sigma$ such that  $x_1 \oplus \gamma > x_1$.
Therefore, 
\begin{equation*}
    (x_1,x_2)\tranLR (x_1\oplus \gamma, x_2 \oplus \gamma) \tranLR (0,y'),
\end{equation*}
by the induction hypothesis.
This completes the proof.
\end{proof}

The following lemma states that if the exponent of the floating-point number $x$ is not too close to $\emin$, then, $(0,x)$ is transferable under $\overline{\mcS}$ to $(0,x^-)$ and $(0,x^+)$.

\begin{lemma}\label[lemma]{lemma:auxilary_large}
Assume that $\sigma$ satisfies \cref{condition:proper_exponent}.
    Consider a floating-point number $0<x\in \fpq$ such that $\emin +2\le \expo{x}< \emax$.
    Then, 
    \begin{equation*}
        (0,x) \tranLR (0, x^+).
    \end{equation*}
    In addition, if $2^{\emin +2}\le x \le 2^{\emax}$, then
        \begin{equation*}
        (0,x) \tranLR (0, x^-).
    \end{equation*}
\end{lemma}
\begin{proof}
By \cref{lem:endbit_control}, there exists $\gamma$ such that $\gamma\in \left(2^{\expo{x}-\mbit-1}, 2^{\expo{x}-\mbit}\right]_{\fpq}\cap \bbS_\sigma$.
    Therefore, 
    \begin{equation*}
        (0,x) \tranLR (\gamma, x\oplus \gamma) = (\gamma, x^+)\tranLR (0,x^+),
    \end{equation*}
    where the last relation is by \cref{lemma:killing_small}.
    Similarly, if $x\neq 2^{\expo{x}}$,
    \begin{equation*}
        (0,x)\tranLR (-\gamma,x\oplus (-\gamma)) = (-\gamma, x^-)\tranLR (0,x^-).
    \end{equation*}
    If $x=2^{\expo{x}}$, there exists $\gamma$ such that $\gamma\in  \left(2^{\expo{x}-\mbit-2}, 2^{\expo{x}-\mbit-1}\right]_{\fpq}\cap \bbS_\sigma$, and similar arguments hold.
    This completes the proof.
\end{proof}
The following lemma states that if the exponent of a floating-point number $x$ is close to $\emin$, then, $(0,x)$ is transferable to $(0,x^-)$ under $\overline{\mcS}$.
Together with \cref{lemma:auxilary_large}, for any $x$ such that $0< x\le 2^{\emax}$, $(0,x)$ can be transferable to $(0,x^-)$ under $\overline{\mcS}$. 
\begin{lemma}\label[lemma]{lemma:auxilary_small_smaller}
Assume that $\sigma$ satisfies \cref{condition:proper_exponent}.
    For $x\in \fpq$ such that $0< x<2^{\emin+2}$, we have
    \begin{equation*}
        (0,x)\tranLR (0,x^-).
    \end{equation*}
\end{lemma}
\begin{proof}[Proof of \cref{lemma:auxilary_small_smaller}]
By \cref{lem:endbit_control}, $\fmin, 2\fmin\in \bbS_\sigma$.
    If $2^{\emin +1}< x<2^{\emin+2}$ and $\mantd{x}{\mbit} = 0$, then,
    \begin{align*}
        &(0, x)\tranLR (\fmin , x\oplus \fmin) = \lrp{\fmin, x}, \\ 
        &\lrp{\fmin, x} \tranLR \lrp{2\fmin, x}
        \tranLR \lrp{0, x\oplus (-2\fmin)} = \lrp{0, x^-}.
    \end{align*}
    If $\expo{x} = \emin +1$ and $\mantd{x}{\mbit} = 1$, then, $\expo{x^-} = \emin+1$ or $\expo{x^+}=\emin+1$.
    If $\expo{x^-} = \emin+1$, then, $\mantd{x}{\mbit}=1$ and thus,
\begin{align*}
    &(0, x)\tranLR (-2\fmin , x\oplus (-2\fmin)) = (-2\fmin, x^-), \\
    & (-2\fmin, x^-)
     \tranLR \lrp{-\fmin, x^- \oplus \fmin}, \\
     &\lrp{-\fmin, x^- \oplus \fmin} = \lrp{-\fmin, x^-}
    \tranLR \lrp{0, x^-}.
\end{align*}
    Symmetric arguments hold for $\expo{x^+} = \emin+1$ case. 
    If $x\le 2^{\emin+1}$, then $x = N\fmin$ where $N\in \left[ 2^{\mbit+1}\right]$ and %
\begin{align*}
&(0, N\fmin)
 \tranLR (\fmin, (N+1)\fmin) \\ &\tranLR \cdots \tranLR  \lrp{\lrp{2^{\mbit+1}-N}\fmin, 2^{\mbit+1}\fmin}
\\ &\tranLR \lrp{\lrp{2^{\mbit+1}-N}\fmin \oplus \fmin, 2^{\mbit+1}\fmin\oplus \fmin} \\
&= \lrp{\lrp{2^{\mbit+1}-N+1}\fmin, 2^{\mbit+1}\fmin}
\\     &\tranLR 
 \lrp{\lrp{2^{\mbit+1}-N+1}\fmin\oplus(-\fmin), 2^{\mbit+1}\oplus(-\fmin)} \\ 
&= \lrp{\lrp{2^{\mbit+1}-N}\fmin, \lrp{2^{\mbit+1}-1}\fmin }
\\    &\tranLR \lrp{\lrp{2^{\mbit+1}-N}\fmin\oplus(-\fmin), \lrp{2^{\mbit+1}-1}\fmin \oplus(-\fmin)} \\
&=  \lrp{\lrp{2^{\mbit+1}-N-1}\fmin, \lrp{2^{\mbit+1}-2}\fmin }
\\ 
&\tranLR \cdots \tranLR \lrp{0, (N-1)\fmin} = \lrp{0,x^-}.
\end{align*}
\end{proof}

\begin{lemma}\label[lemma]{lemma:auxilary_small}
Assume that $\sigma$ satisfies \cref{condition:proper_exponent}.
    Consider $0\neq x\in \fpq$ such that $\expo{x}\le \emax - 2\mbit -2 $.
    Then, there exists $y\in \fpq$ such that $xy>0$, $2^{\emax}\ge |y|>|x|$, and 
    \begin{equation*}
        (0, x)\tranLR (0,y).
    \end{equation*}
\end{lemma}
\begin{proof}
    As the floating-point number is symmetric with respect to zero, we only need to consider the case of $x>0$.
By \cref{condition:proper_exponent}, there exists $K\in \bbS_\sigma$ such that $\expo{K} = \nexpo{x}+\mbit +1$.
If $\mant{K,\mbit}=1$ or $x > 2^{\nexpo{x}}$, 
  \begin{align*}
     (0,x)&\tranLR \lrp{\mant{K}\times 2^{\nexpo{x}+\mbit + 1}, \mant{K}\times 2^{\nexpo{x}+\mbit + 1}\oplus x} \\
     &= \lrp{\mant{K}\times 2^{\nexpo{x}+\mbit + 1}, \lrp{\mant{K}\times 2^{\nexpo{x}+\mbit + 1}}^+}
     \\ &\tranLR \lrp{0, \lrp{\mant{K}\times 2^{\nexpo{x}+\mbit + 1}}^+\ominus \mant{K}\times 2^{\nexpo{x}+\mbit + 1}} \\
    &=\lrp{0, 2^{\nexpo{x}+1}} \text{ or } \lrp{0, 2^{\nexpo{x}+2}}.
 \end{align*}
 If $\mant{K,\mbit}=0$, $x = 2^{\nexpo{x}}$ and $x$ is normal, by \cref{lem:endbit_control}, there exists $\gamma \in \bbS_\sigma$ such that $2^{\expo{x}-\mbit-1}<\gamma \le 2^{\expo{x}-\mbit}$.
 Then,
 \begin{align*}
 &(0,x)\tranLR \lrp{\gamma, x\oplus \gamma } = \lrp{\gamma , x^+} \\
 &\tranLR \lrp{ \mant{K}\times 2^{\nexpo{x}+\mbit + 1}\oplus \gamma , \mant{K}\times 2^{\nexpo{x}+\mbit + 1}\oplus x^+}
  \\&= \lrp{ \mant{K}\times 2^{\nexpo{x}+\mbit + 1}, \lrp{ \mant{K}\times 2^{\nexpo{x}+\mbit + 1}}^+} \\
   &\tranLR \lrp{0, \lrp{\mant{K}\times 2^{\nexpo{x}+\mbit + 1}}^+ \ominus \mant{K}\times 2^{\nexpo{x}+\mbit + 1}} \\&=\lrp{0, 2^{\nexpo{x}+1}} \text{ or } \lrp{0, 2^{\nexpo{x}+2}}.
 \end{align*}
 If $\mant{K,\mbit}=0$, $x = 2^{\nexpo{x}}$ and $x$ is subnormal, by \cref{lem:endbit_control}, $\fmin\in \bbS_\sigma$.
 Therefore, 
 \begin{align*}
 (0,x)&\tranLR (\fmin, \fmin\oplus x) = (\fmin,x^+)
 \\
 &\tranLR \lrp{ \mant{K}\times 2^{\nexpo{x}+\mbit + 1}, \lrp{ \mant{K}\times 2^{\nexpo{x}+\mbit + 1}}^+} \\
 &\tranLR \lrp{0, 2^{\nexpo{x}+1}} \text{ or } \lrp{0, 2^{\nexpo{x}+2}}.
 \end{align*}
 This completes the proof.
\end{proof}

\begin{lemma}\label[lemma]{lemma:two_numbers}
Let $\sigma:\efpq\to\efpq$ and suppose that $\sigma$ satisfies \cref{condition:proper_exponent}.
Then, for any $x_1,x_2\in[-2^{\emax},2^{\emax}]_{\fpq}$ with $x_2-x_1\in(0,2^{\emax}]$, we have 
\begin{equation*}
    (0,x) \tranLR (x_1, x_2).
\end{equation*}
\end{lemma}

\begin{proof}
    We only prove the case $x_1, x_2>0$ and the remaining case is by symmetry. 
We first prove that for any $0<x\in \fpq$ such that $x\in  (0, 2^{\emax})_{\fpq}$,
\begin{equation*}
    (0,x)\tranLR (0,x^+).
\end{equation*}
If $\expo{x}\ge \emin+2$, it is by \cref{lemma:auxilary_large}.
If $\expo{x}\le \emin+1$, then by \cref{lemma:auxilary_small}, there exists $y\in \fpq$ such that $(0,x)\tranLR (0,y)$.
Use mathematical induction on decreasing order to prove that for any $x<x'\le y$, $(0,x)\tranLR (0,x')$.
Assume that if $x'<x''$, then $(0,x)\tranLR (0,x'')$.
By \cref{lemma:auxilary_small_smaller,lemma:auxilary_large}, $(0,x'^+)\tranLR(0,x')$.
Therefore, 
\begin{equation*}
(0,x)\tranLR(0,x'^+)\tranLR(0,x').
\end{equation*}
Therefore, the induction hypothesis holds for any $x'$, which leads to $(0,x^+)$.

Use mathematical induction on the difference $|x_2-x_1|$.
If $x_1 = x_2$, there is nothing to prove.
Assume that the lemma holds for any $x_1', x_2'$ such that $|x_2'-x_1'|<|x_2-x_1|$.
If $x_1< x_2$, by \cref{lemma:auxilary_large,lemma:auxilary_small} and the induction hypothesis (note that $|x_1^+-x_2|<|x_2-x_1|$.)
\begin{equation*}
    (0, x_1)\tranLR (0,x_1^+) \tranLR (0,x_2).
\end{equation*}
Similarly, if $x_1>x_2$, by \cref{lemma:auxilary_small_smaller,lemma:auxilary_large} and the induction hypothesis,
\begin{equation*}
    (0,x_1)\tranLR (0,x_1^-)\tranLR (0,x_2).
\end{equation*}
This completes the proof.
\end{proof}

\begin{lemma}\label[lemma]{lemma:sequential_addition_main}
Let $\sigma:\efpq\to\efpq$ and suppose that $\sigma$ satisfies \cref{condition:proper_exponent}.
Then, for any $y\in [-2^{\emax},2^{\emax})_{\fpq}$ and $x_1, x_2\in [-2^{\emax},2^{\emax}]_{\fpq}$ such that $x_2 - x_1\in (0,2^{\emax}]_{\fpq}$,
we have
\begin{equation*}
    (-2^{\emax},y,y^+, 2^{\emax}) \tranLR  (x_1, x_1, x_2, x_2).
\end{equation*}
\end{lemma}

\begin{proof}

We now prove \cref{lemma:sequential_addition_main}.
By \cref{lemma:nonzero_to_Zero}, there exists $c\in(0,2^{\emax}]_{\fpq}$ and $f_1:(y, y^+)\tranLR (0,c)$.

Then, by the order-preserving property of left-to-right additions, it holds that
$$f_1(x)\in [-2^{\emax},0]_{\fpq},~~f_1(z)\in[c,2^{\emax}]_{\fpq},$$
for all $x\in[-2^{\emax},y]_{\fpq}$ and $z\in[y^+,2^{\emax}]_{\fpq}$.
Furthermore, by \cref{lemma:zero_each_other}, there exists $f_2:(0,c)\tranLR(0,2^{\emax})$.
Again, by the order-preserving property of left-to-right additions, we have
$$f_2\circ f_1(x)\ominus\omega\in[-2^{\emax},-\omega]_{\fpq},~~f_2\circ f_1(z)\ominus\omega=2^{\emax}$$
for all $x\in[-2^{\emax},y]_{\fpq}$ and $z\in[y^+,2^{\emax}]_{\fpq}$.
Here, note that $f_3\defeq f_2\circ f_1\ominus\omega$ is also a left-to-right addition.
We also choose a left-to-right addition $f_4:(-\omega,0)\tranLR(-2^{\emax},0)$ which exists by \cref{lemma:zero_each_other}. Then, it holds that
$$f_4\circ f_3(x)=-2^{\emax},~~f_4\circ f_3(z)=f_4(2^{\emax})\in[0,2^{\emax}]_{\fpq}$$
for all $x\in[-2^{\emax},y]_{\fpq}$ and $z\in[y^+,2^{\emax}]_{\fpq}$.

By \cref{lemma:two_numbers}, there exist $x^*\in[-2^{\emax},2^{\emax}]_{\fpq}$ and a left-to-right addition $f_5:(0,x^*)\tranLR(x_1,x_2)$.
In addition, by \cref{lemma:nonzero_to_Zero,lemma:zero_each_other}, there exists a left-to-right addition $f_6:(-2^{\emax},f_4(2^{\emax}))\tranLR(0,x^*)$.

Then, we have
$$f_5\circ f_6\circ f_4\circ f_3(x)=x_1,~~f_5\circ f_6\circ f_4\circ f_3(z)=x_2$$
for all $x\in[-2^{\emax},y]_{\fpq}$ and $z\in[y^+,2^{\emax}]_{\fpq}$. This implies that $(-2^{\emax},y,y^+,2^{\emax})\tranLR(x_1,x_1,x_2,x_2)$ and completes the proof of \cref{lemma:sequential_addition_main}.

\end{proof}

\section{Technical lemmas on floating-point numbers}

\begin{lemma}\label[lemma]{lemma:distinguishing_point}
    Consider a floating-point $\eta\in \fpq$.
    Then, for any $x_1, x_2\in \fpq$ such that $x_1\neq x_2$, $|x_1|\le |x_2|,$ and 
        \begin{equation*}
        \emin  +1\le \expo{\eta} - \expo{x_2} \le \emax,
    \end{equation*}
    there exist $w,b\in \fpq$ such that
    \begin{equation*}
        \left\{w\otimes x_1\oplus b, w\otimes x_2\oplus b\right\} = \left\{\eta, \eta^+\right\} \text{ or } \left\{\eta^-, \eta^+\right\}.
    \end{equation*}
    Furthermore, we have 
    \begin{equation*}
        |w|\le \lrp{1+ 2^{-\mbit}}\times 2^{\expo{\eta} - \expo{x_2}}, \quad |b|\le \eta^+.
    \end{equation*}
\end{lemma}

\begin{proof}
\textbf{Case 1: $0\le x_1 < x_2$}
   \\Consider the case 
    \begin{equation*}
       0\le x_1< 2^{\nexpo{x_2}}< x_2,
    \end{equation*}
    define $w$ as $w\defeq 2^{\expo{\eta}-\nexpo{x_2}-\mbit-1}$ and $b$ as $b\defeq \eta$.
    Then, 
    \begin{equation*}
        w\otimes x_2 \oplus \eta > 2^{\expo{\eta}-\mbit-1} \oplus \eta = \eta^+,
    \end{equation*}
    and as $w\otimes x_1\le  2^{\expo{\eta}-\mbit-1}$,
    \begin{equation*}
       \eta\le w\otimes x_1 \oplus \eta \le \eta.
    \end{equation*}
    Consider the case 
    \begin{equation*}
         0\le x_1< 2^{\nexpo{x_2}} =  x_2,
    \end{equation*}
    If $\mantd{\eta}{\mbit} =1$, define $w$ as $w\defeq 2^{\expo{\eta}-\nexpo{x_2}-\mbit-1}$ and $b$ as $b\defeq \eta$.
    Then, similar to $x_2>  2^{\expo{x_2}}$ case, $w\otimes x_2 \oplus b = \eta$ and $w\otimes x_1 \oplus b = \eta$.
    If $\mantd{\eta}{\mbit} = 0$, define $w$ as $w\defeq - 2^{\expo{\eta}-\nexpo{x_2}-\mbit-1}$ and $b$ as $b\defeq \eta^+$.
    Then, $w\otimes x_2 \oplus b = \eta$ and $w\otimes x_1 \oplus b = \eta^+$.

    Consider the case $x_1, x_2$ are subnormal or 
    \begin{equation*}
        2^{\expo{x_2}}\le x_1 < x_2.
    \end{equation*}
    Then, there exists $i\in \bbN_0$ such that
    \begin{equation*}
        \mantd{x_1}{j} = \mantd{x_2}{j}  \text{ for } j\in [i],
    \end{equation*}
    and
    \begin{equation*}
        0 = \mantd{x_1}{i+1}\neq  \mantd{x_2}{i+1} =1 .
    \end{equation*}
        If $\mantd{\eta}{\mbit}=1$ 
        define $w\in \fpq$ as 
    \begin{equation*}
        w \defeq 2^{\expo{\eta}+i - \expo{x_2} -\mbit}.
    \end{equation*}
    Define $b$ as 
\begin{equation*}
    b\defeq \eta - 1.\mantd{x_2}{1}\dots \mantd{x_2}{i} \underbrace{0\dots 0}_{\mbit - i}  \times 2^{\expo{\eta}-\mbit +i}.
\end{equation*}
    As 
\begin{equation*}
    1.\mantd{x_2}{1}\dots \mantd{x_2}{i} \underbrace{0\dots 0}_{\mbit - i}  \times w
     = 1\mantd{x_2}{1}\dots \mantd{x_2}{i} \times 2^{\expo{\eta}-\mbit},
\end{equation*}
    the operation is exact: 
\begin{align*}
    \eta - 1.\mantd{x_2}{1}\dots \mantd{x_2}{i} \underbrace{0\dots 0}_{\mbit - i}  \times 2^{\expo{\eta}-\mbit +i}\in \fpq.
\end{align*}
    Then, 
\begin{align*}
    &w \otimes x_1 \oplus b  \\
 &=  1.\mantd{x_2}{1}\dots \mantd{x_2}{i}0 \mantd{x_1}{i+2}\dots \mantd{x_1}{\mbit} \times 2^{\expo{\eta}-\mbit + i} \oplus b \\
     &= \round{\eta + 0. \mantd{x_1}{i+2}\dots \mantd{x_1}{\mbit}\times 2^{\expo{\eta}-\mbit - 1}} = \eta,
\end{align*}
    and 
\begin{align*}
     &w \otimes x_2 \oplus b  \\ 
     &=  1.\mantd{x_2}{1}\dots \mantd{x_2}{i}1 \mantd{x_2}{i+2}\dots \mantd{x_2}{\mbit} \times 2^{\expo{\eta}-\mbit + i} \oplus b \\ 
    &= \round{\eta + 1. \mantd{x_2}{i+2}\dots \mantd{x_2}{\mbit}\times 2^{\expo{\eta}-\mbit - 1}} = \eta^+.
\end{align*}
     If $\mantd{\eta}{\mbit}=0$,
     define $b$ as 
 \begin{equation*}
         b\defeq \eta^+ - 1.\mantd{x_2}{1}\dots \mantd{x_2}{i} \underbrace{0\dots 0}_{\mbit - i}  \times 2^{\expo{\eta}-\mbit +i}.
 \end{equation*}
         and $w$ as $\defeq  -2^{\expo{\eta}+i - \expo{x_2} -\mbit}$.
         Then, similarly, 
\begin{align*}
    &w \otimes x_1 \oplus b  \\
     &=  1.\mantd{x_2}{1}\dots \mantd{x_2}{i}0 \mantd{x_1}{i+2}\dots \mantd{x_1}{\mbit} \times 2^{\expo{\eta}-\mbit + i} \oplus b \\
    &= \round{\eta^+ - 0. \mantd{x_1}{i+2}\dots \mantd{x_1}{\mbit}\times 2^{\expo{\eta}-\mbit - 1}} = \eta^+,
\end{align*}
    and
 \begin{align*}
&w \otimes x_2 \oplus b \\
&=  1.\mantd{x_2}{1}\dots \mantd{x_2}{i}1 \mantd{x_2}{i+2}\dots \mantd{x_2}{\mbit} \times 2^{\expo{\eta}-\mbit + i} \oplus+ b \\ 
&= \round{\eta^+ - 1. \mantd{x_2}{i+2}\dots \mantd{x_2}{\mbit}\times 2^{\expo{\eta}-\mbit - 1}} = \eta.
\end{align*}

    \textbf{Case 2: $x_1< 0 < x_2$}
    Define $b$ as $b\defeq \eta$.
    If $\nmant{x_2} < 1.1$, define $w$ as 
    \begin{equation*}
        w\defeq - 2^{\expo{\eta} - \nexpo{x_2} - \mbit}.
    \end{equation*}
    Then, 
    \begin{equation*}
        w\otimes x_2 = \nmant{x_2} \times 2^{\expo{\eta} -\mbit},
    \end{equation*}
    and 
    \begin{equation*}
        w\otimes x_2 \oplus b = \eta^-.
    \end{equation*}
    And as 
    \begin{equation*}
       | w\otimes x_2 | \le | w\otimes x_1 |,
    \end{equation*}
\begin{equation*}
    \eta\le w\otimes x_2 \oplus b\le \eta^+.
\end{equation*}
 If $\nmant{x_2} \ge 1.1$, define $w$ as 
    \begin{equation*}
        w\defeq - 2^{\expo{\eta} - \nexpo{x_2} - \mbit -1},
    \end{equation*}
and we get the same conclusion.
    This completes the proof.
\end{proof}

\begin{lemma}\label[lemma]{lem:endbit_control}
    Suppose that $\sigma:\efpq\to\efpq$ satisfies \cref{condition:proper_exponent} with constants $C_1$ and $C_2$ and let {$e\in [\emin -\mbit,\emax-\mbit]_{\bbZ}$}.
    Then, there exists $\gamma\in \fpq$ and $i\in [2]$ such that 
    \begin{equation*}
   2^{e-1}   < \gamma\otimes \sigma(C_i) \le 2^{e}.
    \end{equation*}
\end{lemma}

\begin{proof}
    
    Let $C_1$ and $C_2$ from \cref{condition:proper_exponent} be represented as 
    \begin{equation*}
        \sigma(C_1) = \mant{\sigma(C_1)}\times 2^{\expo{\sigma(C_1)}} \quad \text{ and }   \sigma(C_2) = \mant{\sigma(C_2)}\times 2^{\expo{\sigma(C_2)}}. 
    \end{equation*}

    Consider the case $\emin -\mbit -2\le e\le 2$.
    If $\mant{\sigma(C_1)}\in \left[1, \frac{5}{4}\right]_{\fpq}$, define $\gamma$ as 
    \begin{equation*}
        \gamma \defeq 1.1\times 2^{e - \expo{\sigma(C_1)}-1}.
    \end{equation*}
    As $ \expo{\sigma(C_1)}\ge \emin = -2^{\ebit-1} +2$, $e - \expo{\sigma(C_1)}-1 \le 2^{\ebit-1}-1 = \emax$, and $\gamma\in\fpq$.
    Then, 
    \begin{equation*}
       \frac{3}{4} \times  2^{e}\le \gamma \times \sigma(C_1) = \mant{\sigma(C_1)}\times 0.11 \times 2^{e} \le \frac{15}{16}\times 2^{e},
    \end{equation*}
    and thus
    \begin{equation*}
        \frac{3}{4}\times 2^{e}\le 
        \gamma \otimes \sigma(C_1) = \round{\mant{\sigma(C_1)}\times 1.1\times  2^{e}}_{\fpq}
        \le  2^{e},
    \end{equation*}
    where the first inequality is satisfied as $e\ge \emin-\mbit+2$.
    If $\mant{\sigma(C_1)}\in \left(\frac{5}{4},2\right)_{\fpq}$, define $\gamma$ as 
    \begin{equation*}
        \gamma \defeq 2^{e - \expo{\sigma(C_1)}-1}.
    \end{equation*}
    Similar to the above case, $\gamma\in\fpq$.
    Then,
    \begin{equation*}
      \frac{1}{2}\times 2^{e} <  \gamma \otimes \sigma(C_1) = \round{\mant{\sigma(C_1)}\times 2^{e -1}}_{\fpq} \le 2^{e },
    \end{equation*}
     where the first inequality is satisfied as $e-1\ge \emin-\mbit+1$.

    Consider the case of $2< e <\emax-\mbit$.
   If $\mant{\sigma(C_2)}\in \left[1, \frac{5}{4}\right]_{\fpq}$, define $\gamma$ as 
    \begin{equation*}
        \gamma \defeq 1.1\times 2^{e - \expo{\sigma(C_2)}-1}.
    \end{equation*}
    As $e - \expo{\sigma(C_2)}-1\le {\emax} -\mbit - 1 - \expo{\sigma(C_2)}-1\le \emax$, $\gamma\in \fpq$.
        Then, similar to the above case, 
        \begin{equation*}
            \frac{3}{4} \times  2^{e}\le \round{\gamma \times \sigma(C_2)}\le  2^{e}.
        \end{equation*}
      If $\mant{\sigma(C_2)}\in \left(\frac{5}{4},2\right)_{\fpq}$, define $\gamma$ as 
    \begin{equation*}
        \gamma \defeq 2^{e - \expo{\sigma(C_2)}-1},
    \end{equation*}
    and similar arguments holds.
    
    Consider the case $e=\emax -\mbit$.
     If $\mant{\sigma(C_2)}\ge 1^{++}$, define $\gamma$ as 
    \begin{equation*}
       \gamma\defeq  \lrp{2-2^{-\mbit}} \times 2^{\emax - \mbit -2 - \expo{\sigma(C_2)}}.
    \end{equation*}
    Then, $\gamma\in \fpq$, and we have
    \begin{align*}
   2^{\emax-\mbit-1} &<     \gamma \otimes \sigma(C_2) \\
   &= \lrp{\mant{\sigma(C_2)} \otimes \lrp{2-2^{-\mbit}}}\times 2^{\emax - \mbit -2} \le 2^{\emax-\mbit}.
    \end{align*}
     If $\mant{\sigma(C_2)}\le  1^{+}$, then, $\expo{\sigma(C_2)}\ge -\mbit -1$. 
    If $\mant{\sigma(C_2)}\in \left[1, \frac{5}{4}\right]_{\fpq}$, define $\gamma$ as 
       \begin{equation*}
        \gamma \defeq 1.1\times 2^{\emax-\mbit -1 - \expo{\sigma(C_2)}},
    \end{equation*}
    and if $\mant{\sigma(C_2)}\in \left(\frac{5}{4},2\right)_{\fpq}$, define $\gamma$ as 
    \begin{equation*}
        \gamma \defeq 2^{\emax-\mbit-1- \expo{\sigma(C_2)}}.
    \end{equation*}
    Then, $\gamma\in \fpq$, and similar arguments hold as in $2< e <\emax-\mbit$ case.
    
     Consider the case $e = \emin -\mbit +1$.
     If $1\le \mant{\sigma(C_1)}<\frac{5}{4}$, define $\gamma$ as 
     \begin{equation*}
         \gamma\defeq 2^{e-\expo{\sigma(C_1)}}.
     \end{equation*}
     Then,
     \begin{equation*}
         \gamma\otimes \sigma(C_1) =  \round{\mant{\sigma(C_1)}\times 2^{e}}_{\fpq}
         =  \round{\mant{\sigma(C_1)}\times 2\fmin}_{\fpq} = 2\fmin = 2^{e}.
     \end{equation*}
     If $\frac{5}{4}\le \mant{\sigma(C_1)} < \frac{5}{3}$, define $\gamma$ as 
     \begin{equation*}
        \gamma\defeq  1.1\times 2^{e-\expo{\sigma(C_1)}-1}.
     \end{equation*}
     As $\expo{\sigma(C_1)}\le -1$, $\gamma\in \fpq$. 
          Then,
\begin{align*}
 \gamma\otimes \sigma(C_1) &= \round{\mant{\sigma(C_1)}\times 1.1\times 2^{e-1}}_{\fpq} \\
 &=  \round{\mant{\sigma(C_1)}\times 1.1 \times \fmin}_{\fpq} = 2\fmin = 2^{e}.
\end{align*}
          If $\frac{5}{3}\le \mant{\sigma(C_1)} <2 $, define $\gamma$ as 
     \begin{equation*}
        \gamma\defeq  2^{e-\expo{\sigma(C_1)}-1}.
     \end{equation*}
         Then,
     \begin{equation*}
         \gamma\otimes \sigma(C_1) =  \round{\mant{\sigma(C_1)}\times  2^{e-1}}_{\fpq}
         =  \round{\mant{\sigma(C_1)}\times  \fmin}_{\fpq} = 2\fmin = 2^{e}.
     \end{equation*}

     Consider the case $e = \emin -\mbit$.
     If $1 \le \mant{\sigma(C_1)} <1.1 $, define $\gamma$ as 
    \begin{equation*}
        \gamma\defeq  2^{e-\expo{\sigma(C_1)}}.
    \end{equation*}
      Then,
     \begin{equation*}
         \gamma\otimes \sigma(C_1) =  \round{\mant{\sigma(C_1)}\times  2^{e}}_{\fpq}
         =  \round{\mant{\sigma(C_1)}\times  \fmin}_{\fpq} = \fmin = 2^{e}.
     \end{equation*}
          If $1.1 \le \mant{\sigma(C_1)} <2 $, define $\gamma$ as 
    \begin{equation*}
        \gamma\defeq  2^{e-\expo{\sigma(C_1)}-1}.
    \end{equation*}
    As $\expo{\sigma(C_1)}\le -1$, $\gamma\in \fpq$. 
      Then,
\begin{align*}
 \gamma\otimes \sigma(C_1) &=  \round{\mant{\sigma(C_1)}\times  2^{e}-1}_{\fpq} \\
 &=  \round{\mant{\sigma(C_1)}\times \frac{1}{2} \fmin}_{\fpq} = \fmin = 2^{e}.
\end{align*}
    
This completes the proof.

\end{proof}

\section{Technical lemmas on reduction order }

We present several technical lemmas on reduction order.

\begin{lemma}\label[lemma]{lemma:reduction_tree}
For every $n\ge 1$ and every $o\in\mcO_n$, there exists a full binary tree
$T_o$ with exactly $n$ leaves, whose leaves are labeled by
$x_1,\dots,x_n$, such that $\Sigma(\bfx;o)$ is obtained by evaluating
$T_o$ with the binary operation $\oplus$ at every internal node.
In particular, every reduction order $o\in\mcO_n$ determines a full binary
addition tree for $n$ inputs.
\end{lemma}

\begin{proof}
We prove the claim by induction on $n$.

For $n=1$, the reduction order is $(1)$, and $\Sigma(x_1;(1))=x_1$.
The corresponding tree consists of a single leaf labeled by $x_1$.

For $n=2$, the only reduction order is $(1,2)$, and
\[
\Sigma((x_1,x_2);(1,2))=x_1\oplus x_2.
\]
Thus the corresponding tree has one internal node with two leaves labeled
$x_1$ and $x_2$.

Assume the claim holds for $n-1$. Let
\[
o=((i_1,j_1),\dots,(i_{n-1},j_{n-1}))\in\mcO_n.
\]
By definition, the first reduction replaces $x_{i_1}$ and $x_{j_1}$ by
$x_{i_1}\oplus x_{j_1}$ and produces the $(n-1)$-tuple
$\widehat{\bfx}_{i_1,j_1}$. Let
\[
o'=((i_2,j_2),\dots,(i_{n-1},j_{n-1}))\in\mcO_{n-1}.
\]
By the induction hypothesis, $o'$ determines a full binary tree for
$\Sigma(\widehat{\bfx}_{i_1,j_1};o')$. Replacing the leaf corresponding to
the entry $x_{i_1}\oplus x_{j_1}$ by an internal node with two children
labeled $x_{i_1}$ and $x_{j_1}$ gives a full binary tree with $n$ leaves.
Evaluating this tree gives exactly
\[
\Sigma(\bfx;o)
=
\Sigma(\widehat{\bfx}_{i_1,j_1};o').
\]
Therefore $o$ determines a full binary addition tree. This completes the
induction.
\end{proof}

\begin{definition}[Depth of a reduction order]
Let $o\in\mcO_n$, and let $T_o$ be the full binary tree corresponding to
the reduction order $o$ as in Lemma~\ref{lemma:reduction_tree}.

For a node $v$ of $T_o$, define its depth as the number of edges on the
unique path from the root to $v$.
The depth of the tree $T_o$ is
\[
\depth(T_o)
\defeq
\max_{v\in \mathrm{Leaf}(T_o)}
\depth(v).
\]

We define the depth of the reduction order $o$ by
\[
\depth(o)
\defeq
\depth(T_o).
\]
\end{definition}
Note that 
\[
\depth\bigl(((1,2),(1,2),\dots,(1,2))\bigr)=n-1,
\]
and 
\[
\depth\bigl(((1,2),(2,3),(1,2))\bigr)=2.
\]

\begin{lemma}\label[lemma]{lem:binary_tree}
    Let $o \in \mcO_n$. Then 
    \begin{align*}
      \ceil{\log_2(n)}   \le \mathrm{depth}(o) \le n-1.
    \end{align*}
\end{lemma}
\begin{proof}
Let $T_o$ be the full binary tree corresponding to the reduction order $o$.
By definition,
\[
\depth(o)=\depth(T_o).
\]

Since $T_o$ has exactly $n$ leaves and every internal node has exactly two
children, a binary tree of depth $d$ can have at most $2^d$ leaves.
Therefore
\[
n \le 2^{\depth(T_o)}.
\]
Taking $\log_2(\cdot)$ gives
\[
\log_2 n \le \depth(T_o).
\]
Since $\depth(T_o)$ is an integer, we obtain
\[
\ceil{\log_2 n}\le \depth(T_o)=\depth(o).
\]

It remains to prove the upper bound. Any full binary tree with $n$ leaves
has exactly $n-1$ internal nodes. Every path from the root to a leaf contains
at most all internal nodes of the tree. Hence the number of edges on such a
path is at most $n-1$. Therefore
\[
\depth(T_o)\le n-1.
\]
Thus
\[
\ceil{\log_2 n}\le \depth(o)\le n-1.
\]
\end{proof}

\begin{lemma}\label[lemma]{lem:seq_to_arbi}
Let $\bfx \in \fpq^n$ and $\mcS = \{o_{2}, o_3, \dots , \}$ be reduction order. 
There exist $ k \in \bbN$ with $k \ge n$ and   embedding $\iota : \fpq^n \to \fpq^k$ such that 
\begin{align*}
    \Sigma(\bfx, \bar{o}_n) = \Sigma( \iota(\bfx) , o_k). 
\end{align*}

\end{lemma}

\begin{proof}
We first note that every non-leaf node in a full binary tree has two children, and hence each non-root node has exactly one sibling.
From the sequence $ o_n , o_{n+1} , \dots ,  $, pick the smallest $k \in \bbN$ such that $\mathrm{depth}(o_k) \ge n-1$. Such $k \in \bbN$ exists by \cref{lem:binary_tree}. Let $\mathcal{T}_k$ be the binary tree for $\Sigma( \tilde{\bfx}_k , o_k)$ and $v_0 , \dots v_{n-1}$ be a sequence of nodes from root $v_0$ to the leaf $v_{n-1}$. Let $v_n$ be the sibling of $v_{n-1}$ ($v_{n-2} = v_{n-1} \oplus v_{n})$. Let $w_n = v_n, w_{n-1} = v_{n-1}$. For $i= n-2 , n-3, \dots , 1$, let $w_i$ be the leftest leaf of the sibling of $v_i$. Then for leaves $\lambda_1, \dots , \lambda_k$ of $\mathcal{T}_k$, we define
\begin{align*}
    \lambda_i = \begin{cases}
        x_{n+1-j} \; &\text{if} \; \lambda_i = w_j \\ 
        0 \; &\text{if} \; \lambda_i \ne w_j 
    \end{cases}
\end{align*}
We define $\iota : \fpq^n \to \fpq^k$ such that \[ \iota(\bfx) = (\lambda_1 , \dots, \lambda_k). \]
Then we have 
\begin{align*}
    \Sigma( \iota(\bfx) , o_k) &= (\dots  (x_1 \oplus x_2 ) \oplus x_3 )\oplus \dots ) \oplus x_n, \\
    &=  \Sigma(\bfx, \bar{o}_n).
\end{align*}
\end{proof}

\begin{lemma}\label[lemma]{lemma:transfer_seq_to_reduction}
Let $\mathcal{S}$ be a reduction order.
Suppose 
\begin{align*}
    (x_1, \dots, x_m)  \xRightarrow{(\sigma, \bfz, 0 ,  \overline{\mcS})}  (y_1,\dots, y_m),
\end{align*}
for $\bfz = (z_1 , \dots z_n ) \in \bbS_\sigma^n$.Then there exist $\tilde{n} \ge n$, $j \in \{0\} \cup[\tilde{n}]$, and  $\tilde{\bfz} \in \bbS^{\tilde{n}}$ such that  
\begin{align*}
    (x_1, \dots, x_m)  \xRightarrow{(\sigma,\tilde{\bfz}, j, \mcS)}  (y_1,\dots, y_m).
\end{align*}     
\end{lemma}
\begin{proof}
By assumption, we have 
\begin{align*}
    \Sigma( (x_i , z_1 , \dots , z_n) , \overline{S}) =  \Sigma( (x_i , z_1 , \dots , z_n) , \bar{o}_{n+1}) =  y_i. 
\end{align*}
By \cref{lem:seq_to_arbi}, we have embedding $\iota : \fpq^{n+1} \to \fpq^k$ such that 
\begin{align*}
    \Sigma(\bfx, \bar{o}_{n+1}) = \Sigma( \iota(\bfx) , o_k), 
\end{align*}
for $\bfx \in \fpq^{n+1}$.

Let $j_a$ denote the position of embedding of $x_a$ via $\iota$ : 

\[ \iota( (x_1 , \dots , x_{n+1}) ) = 
(\dots , \; \underset{\substack{\uparrow \\ \scriptstyle j_a\text{-th position}}}{x_a} , \; \dots ) \in \fpq^k.
\]

Define $\iota_2 : \fpq \to \fpq^{k}$ such that 
\begin{align*}
    \iota_2(x) &= \iota( (x , z_1 , \dots , z_n)) \\
    &=  (\dots , \; \underset{\substack{\uparrow \\ \scriptstyle j_1\text{-th position}}}{x} , \; \dots , \;   \underset{\substack{\uparrow \\ \scriptstyle j_{b+1}\text{-th position}}}{z_b} , \dots ) \ 
\end{align*}
for $b \in [n]$. 
Define $\tilde{z} \in \fpq^{k-1}$ by removing $x$ in $\iota_2(x)$ :  
\begin{align*}
    \tilde{z} &=   ( \dots , \; \underset{\substack{\uparrow \\ \scriptstyle j_1\text{-th position}}}{ \hat{x}} \; , \dots ) \; \text{in $\iota_2(x)$ }.
\end{align*}

Hence we have 
\begin{align*}
   \Sigma(\iota_{\tilde{z}, j_1}(x) , \mcS) =  \Sigma( \iota_2(x_i) , \mcS) =  \Sigma( (x_i ,z_1 , \dots, z_n), \overline{\mcS}) =  y_i.
\end{align*}
for $i \in [m]$.
\end{proof}

\begin{lemma}\label[lemma]{lemma:sequential_addition_main_reduction_order}
Let $\sigma:\efpq\to\efpq$ and suppose that $\sigma$ satisfies \cref{condition:proper_exponent}.
Then, for any $y\in [-2^{\emax},2^{\emax})_{\fpq}$ and $x_1, x_2\in [-2^{\emax},2^{\emax}]_{\fpq}$ such that $x_2 - x_1\in (0,2^{\emax}]_{\fpq}$, 
there exist $\bfz = (z_1 , \dots z_n ) \in \bbS_\sigma^n$ and $j \in \{0\} \cup [n]$ such that   
\begin{equation*}
    (-2^{\emax},y,y^+, 2^{\emax}) \xRightarrow{(\sigma, \bfz, j, \mcS)}  (x_1, x_1, x_2, x_2).
\end{equation*}
\end{lemma}
\begin{proof}
    By \cref{lemma:sequential_addition_main}, 
there exist $n' \in \bbN, \bfz' \in \bbS_\sigma^{n'}$ such that   
\begin{equation*}
(-2^{\emax},y,y^+, 2^{\emax}) \xRightarrow{(\sigma, \bfz', 0 , \overline{\mcS})}  (x_1, x_1, x_2, x_2).
\end{equation*}
By \cref{lemma:transfer_seq_to_reduction}, there exist $n \ge n'$, $j \in \{0\} \cup [n]$ and  $\bfz \in \bbS^{n}$ such that  
\begin{align*}
    (x_1, \dots, x_m)  \xRightarrow{(\sigma,\bfz, j,  \mcS)}  (y_1,\dots, y_m).
\end{align*}     

\end{proof}

\subsection{Proof of \cref{thm:main}}
\label{sec:pflem:thm_main}
To prove \cref{thm:main}, we present the following lemmas with proofs presented in \cref{subsec:pflem:indicator_from_distinguishable,sec:pflem:last_layer_sums}  

\begin{lemma}\label[lemma]{lemma:indicator_from_distinguishable}
Let $\sigma:\efpq\to\efpq$, $d\in\bbN$, and $\domain\subset \fpq^d$. %
Suppose that $\sigma$ satisfies \cref{condition:proper_exponent}
and $\domain$ is $\sigma$-distinguishable with range $\left[-2^{\emax }, 2^{\emax }\right]_{\fpq}$ under $\mcS$.
Then, for any $z\in \fpq^d$ and $c\in\{C_1,C_2\}$, %
there exists a three-layer $\sigma$ network $f:\domain\rightarrow\fpq$ ending with the activation function such that 
    \begin{equation*}
        f(x) = {\sigma}\lrp{c}\indc{z}{x}.
    \end{equation*}
\end{lemma}

\begin{lemma}\label[lemma]{lem:last_layer_sums}
    Suppose that $\sigma:\efpq\to\efpq$ satisfies \cref{condition:proper_exponent}.
    Then, for any $x\in\fpq\cup\{-\infty,\infty\}$, there exist $n\in\bbN$, $w_1,\dots,w_n\in\fpq$, and $z_1,\dots,z_n\in\{C_1,C_2\}$ such that
    \begin{equation*}x=(w_1\otimes\sigma(z_1))\oplus\cdots\oplus(w_n\otimes\sigma(z_n)).
    \end{equation*}
\end{lemma}

We now prove \cref{thm:main} using \cref{lemma:indicator_from_distinguishable,lem:last_layer_sums}. Without loss of generality, we assume $d_2=1$, i.e., the target function $f$ is scalar-valued. To represent $f':\domain\to(\fpq\cup\{-\infty,\infty\})^{d_2}$ with $d_2>1$, we can construct $d_2$ networks that represent coordinatewise functions $x\mapsto f'(x)_i$ and concatenate them.

For any $f:\domain\to\fpq\cup\{-\infty,\infty\}$, we can represent $f$ as follows:
\begin{equation*}
    f(x) = \sum_{y\in \domain} f(y)\indc{y}{x}.
\end{equation*}
Let $c\defeq f(y)$.
Then, by \cref{condition:proper_exponent} and \cref{lem:last_layer_sums}, for any $c \in \fpq\cup\{-\infty,\infty\}$, there exist $n_c\in\bbN$, $z_{c,1},\dots, z_{c, n_c}\in\{C_1,C_2\}$, and $w_{c,1},\dots, w_{c, n_c}\in\fpq$ such that
\begin{equation*}
    c = \Sigma( (w_{c,1} \otimes \sigma(z_{c,1}) ) , \dots , (w_{c,n_c} \otimes \sigma(z_{c,n_c}) ) , \overline{\mcS})
\end{equation*}
By \cref{lemma:indicator_from_distinguishable}, for each $f(y)=c\in\fpq$ and $i\in[n_c]$, there exists a three-layer $\sigma$-network $h_{c,i}:\domain\rightarrow \fpq$ ending with the activation function such that 
    $h_{c,i}(x) = \sigma(z_{c,i})\indc{y}{x}.$ 
For $\domain = \{ y_1 , \dots , y_{|\domain|} \}$, we construct the target four-layer $\sigma$ network $\bar{g}$ as follows:
\begin{align*}
    & \bar{g}(x) = \\
    &(w_{f(y_1),1}\otimes h_{f(y_1),1}) \oplus \dots \oplus (w_{f(y_1),n_{f(y_1)}}\otimes h_{f(y_1),n_{f(y_1)}}) \oplus \\ & \dots \\ &\oplus (w_{f(y_{1}),n_{y_{1}} }\otimes h_{f(y_{1}),n_{y_{1}} })\oplus \dots \\
    &\oplus (w_{f(y_{|\domain|}),n_{y_{|\domain|}} }\otimes h_{f(y_{|\domain|}),n_{y_{|\domain|}} }).
\end{align*}

Note that $\bar{g}$ is left-to-right addition and $\bar{g}(x)= f(x)$. By \cref{lem:seq_to_arbi}, there exists $g$ such that 
\[ g(x) = \bar{g}(x),\]
and $g(x)$ under $\mcS$.

This completes the proof of \cref{thm:main}.

\subsection{Proof of \cref{lemma:indicator_from_distinguishable}}\label{subsec:pflem:indicator_from_distinguishable}

By the definition of the distinguishability, there exist $n\in \bbN$ and affine transformations \smash{$\phi_{1},\dots,\phi_{n}:\efpq^d\rightarrow\efpq$} under $\mcS$
satisfying the following two properties: \\
(i) for any $y\in \domain$, there exists $j_y\in [n]$ such that
$\sigma(\phi_{j_y}(z)) \neq 
\sigma(\phi_{j_y}(y))$, \\(ii) $\sigma\lrp{\phi_{j}(\domain)}\subset \left[-2^{\emax}, 2^{\emax }\right]_{\fpq}$ for all $j\in[n]$.
 
Consider $C_0, C_1\in \fpq$ in \cref{condition:proper_exponent}.
Without loss of generality, we assume $C_0<C_1$, $\sigma(C_1)>0$, and $C_0<c$; the proof for the remaining  cases can be done similarly.

By \cref{lemma:sequential_addition_main}, there exist $\bfz_{j,k}= (z_{j,k,1}, \dots , z_{j,k,n_{j,k}})$ for $k \in [2]$ such that 
\begin{align*}
  x \xRightarrow{(\sigma, \bfz_{j,1} , 0 ,  \overline{\mcS})}  
   & \begin{cases}
C_1 &\text{ if }  \sigma\lrp{\phi_{j}(z)} \le x \le 2^{{\emax}},
\\ C_0 &\text{ if } -2^{{\emax} } \le x <\sigma\lrp{\phi_{j}(z)},
\end{cases}\\
 -x \xRightarrow{(\sigma, \bfz_{j,2} , 0 ,  \overline{\mcS})}  
 & \begin{cases}
C_0 &\text{ if }\sigma\lrp{\phi_{j}(z)} < x \le 2^{{\emax} },
\\ C_1 &\text{ if } -2^{{\emax} } \le x \le \sigma\lrp{\phi_{j}(z)}.
\end{cases}
\end{align*}
By \cref{lemma:transfer_seq_to_reduction}, there exist $l_{j,k}$ and  $\tilde{\bfz}_{j,k}$ such that 
\begin{align*}
  x \xRightarrow{(\sigma, \tilde{\bfz}_{j,1} , l_{j,1},  \mcS)}  
   & \begin{cases}
C_1 &\text{ if }  \sigma\lrp{\phi_{j}(z)} \le x \le 2^{{\emax}},
\\ C_0 &\text{ if } -2^{{\emax} } \le x <\sigma\lrp{\phi_{j}(z)},
\end{cases}\\
 -x \xRightarrow{(\sigma, \tilde{\bfz}_{j,2} , l_{j,2} ,  \mcS)}  
 & \begin{cases}
C_0 &\text{ if }\sigma\lrp{\phi_{j}(z)} < x \le 2^{{\emax} },
\\ C_1 &\text{ if } -2^{{\emax} } \le x \le \sigma\lrp{\phi_{j}(z)}.
\end{cases}
\end{align*}

Define $g_{j,k}:\domain\rightarrow \fpq$ as 
\begin{align*}
g_{j,k}(x) & \defeq \sigma(g_{j,k}^*), \quad     g_{j,k}^* = \Sigma((x, \tilde{z}_{j,k,1}, \dots , \tilde{z}_{j,k,n_{j,k}}), \mcS).
 \end{align*}

Then, one can observe that $g_{j,k}$ is a two-layer network ending with the activation function.
Furthermore, for any $x\in\domain$, $g_{j,k}(x)=\sigma(C_1)$ for all $j\in[n]$ and $k\in\{1,2\}$ if and only if $x=z$; otherwise, there exists some $j,k$ such that $g_{j,k}(x)=\sigma(C_0)=0$.

By \cref{lemma:sequential_addition_main}, we define $h_1$, $h_2,\dots,h_{2n-1}$ and $h$ such that for $i\in[2n-2]$,
\begin{align*}
    h_i &: (0,\sigma(C_1),2\sigma(C_1)) \tranLR (0,0,\sigma(C_1)), \\
    h_{2n-1} &: (0,\sigma(C_1),2\sigma(C_1)) \tranLR (C_0,C_0,c), \\
    h(y_1,\dots,y_{2n})&\defeq
h_{2n-1}\lrp {\cdots h_2\lrp{h_1(y_1\oplus  y_2)\oplus y_3} \dots \oplus  y_{2n}}.
\end{align*}
Then we have 
\begin{align*}
    h(g_{1,1}(x), g_{1,2}(x), \dots, g_{n,1}(x), g_{n,2}(x)) =  c \indc{z}{x}. 
\end{align*}
Since $h(y_1 , \dots , y_{2n})$ is left-to-right addition, by \cref{lem:seq_to_arbi}, there 
exist $\tilde{h}$ such that $\tilde{h}$ is addition under $\mcS$ and $ \tilde{h}(y_1, \dots , y_{2n}) = h(y_1 , \dots , y_{2n} )$.

We now construct the target three-layer network $f$ ending with the activation function $\sigma$ under $\mcS$ as
\begin{align*}
&f(x)\defeq\sigma(\tilde{h}(g_{1,1}(x),g_{1,2}(x),\dots,g_{n,1}(x),g_{n,2}(x))).
\end{align*}
Then we have 
\begin{align*}
    f(x) =  \sigma(c) \indc{z}{x}. 
\end{align*}
 
This completes the proof of \cref{lemma:indicator_from_distinguishable}  which leads to the proof of \cref{thm:main}.

\subsection{Proof of \cref{lem:last_layer_sums}}\label{sec:pflem:last_layer_sums}
    By \cref{lem:endbit_control}, for any $e\in [\emin-\mbit,\emax-\mbit]$, there exists $\gamma$ and $i\in [2]$ such that
    \begin{equation*}
        2^{e-1}<\gamma\otimes \sigma(C_i)\le 2^{e}.
    \end{equation*}
    Therefore, for any $x\in\fpq$, there exists $\gamma$ and $i\in [2]$ such that $x\oplus \gamma\otimes \sigma(C_i) = x^+$.
    This completes the proof.

\section{Proofs of theorems and lemmas on necessary conditions of activation functions}

\subsection{Proof of \cref{lem:necessary}}\label{sec:pflem:necessary}
    Suppose that $\domain$ is not $\sigma$-distinguishable with range $\fpq$ under $\mcS$. 
    Then, there exist $x_1, x_2\in \domain$ such that for any $d_2\in \bbN$ and affine transformations  $\phi_1,\phi_2:\efpq^{d}\rightarrow \efpq^{d_2}$ under $\mcS$, 
    \begin{equation*}
        \sigma(\phi_1(x_1) ) = \sigma(\phi_1(x_2)) .
    \end{equation*}
    This implies that for any $\sigma$ network $g:\domain\rightarrow \efpq$, $g(x_1) = g(x_2).$
    In other words, $\sigma$ networks cannot represent a function $f:\domain\rightarrow\fpq$ with $f(x_1)\neq f(x_2)$. This completes the proof.
\subsection{Proof of \cref{lem:necc_example}}
\label{sec:pflem:necc_example}    

Note that $|\hatsig(x) - \hatsig(0)| \le 2^{-2\mbit-1}$ for $|x| \le 2^{-\mbit-1}$, which implies $\round{\hatsig(x)} = \round{\hatsig}(0)$ for $|x| \le 2^{-\mbit-1}$. Suppose that there exist $w, b \in \fpq$ such that 
$$\round{\hatsig}( w \otimes 0 \oplus b) \ne  \round{\hatsig}( w \otimes \fmin \oplus b).$$
Then, $|b| > 2^{-\mbit-1}$ and $| w \otimes \omega | \ge 2^{-2\mbit-1}$, which implies
\begin{align*}
    |w| \ge  2^{-\emin -\mbit-1}.
\end{align*}
Hence, for $|x| \ge  2^{\mbit+3} $, we have 
\begin{align*}
    |w \otimes x  \oplus b | &\ge 2^{\emax + 1}~~\text{or}~~
    |w \otimes ( - x ) \oplus b | &\ge 2^{\emax + 1},
\end{align*}
leading to $ |w \otimes x  \oplus b |  = \infty $ or $ |w \otimes (-x)  \oplus b |  = \infty$.

\subsection{Proof of \cref{cor:necessary}}\label{sec:pfcor:necessary}

We use \cref{lem:necc_example}. 
Note that $ (1+2^{-\mbit}) \cdot 2^{-\mbit} \le \tfrac{5}{16}$ for $\mbit \ge 2$.
Let $\rho_1(x) = \cos(x), \rho_2(x) = 1+x^2, \rho_3(x) = \cosh(x), \rho_4(x) = 2^{-x^2/2}$.
Then we have
\begin{align*}
    \rho_1(0) = 1 > \tfrac{5}{16}, \;  \rho_2(0)= 1 >\tfrac{5}{16}, \\
    \rho_3(0) = 1  > \tfrac{5}{16}, \; \rho_4(0)= 1 > \tfrac{5}{16},
\end{align*}
and for $|x| \le 2^{-\mbit-1}$
\begin{align*}
    |\rho_1'(x)| &= |-\sin(x)| \le |x| \le 2^{-\mbit}, \\  |\rho_2'(x)| &= 2x \le 2^{-\mbit} , \\
    |\rho_3'(x)| &= |\sinh(x)| \le  |2x| \le 2^{-\mbit}, \\
    |\rho_4'(x)| &= |-x e^{-x^2/2}| \le |x| \le 2^{-\mbit} ,
\end{align*}
since $|\sinh(x)| \le |2x|$ for $x \in [0,2.1]$ and $|e^{-x^2/2}| \le 1 $.

By \cref{lem:necc_example}, $f(0)\ne f(\fmin)$ cannot be represented by a $\round{\rho}$ network under $\mcS$.

\section{Proofs of theorems and lemmas on sufficient conditions of activation functions}

\subsection{Proof of \cref{lemma:distinguishing_points_to_distinguishable}}\label{sec:pflem:distinguishing_points_to_distinguishable}
     Consider arbitrary $x_1, x_2\in \left(-2^{e_2+1}, 2^{e_2+1}\right)_{\fpq}$ such that $x_1\neq x_2$ and $|x_1|\le |x_2|$.
     By the condition of the lemma, there exists $i\in [n]$ such that
     \begin{equation*}
         \expo{x_2}\in \left[\expo{\eta_i}-e_1, \expo{\eta_i} +\emax-2\right]_{\bbZ}.
     \end{equation*}
     By \cref{lemma:distinguishing_point},
     there exist $w$ and $b$ such that
    \begin{equation*}
        \sigma(w\otimes x_1 \oplus b)\neq \sigma(w\otimes x_2 \oplus b),
    \end{equation*}
    \begin{equation*}
        |w|\le \lrp{1+2^{-\mbit}}\times 2^{e_1}, \text{ and } |b|\le |\eta_i|^+.
    \end{equation*}
    Then,
    \begin{align*}
        &\sigma\lrp{\left(-2^{e_2+1}, 2^{e_2+1}\right)_{\fpq}\otimes w\oplus b}
        \\
        & \; \subset \sigma\lrp{\left[-2^{e_2+e_1 + 1}, 2^{e_2+e_1 + 1}\right]_{\fpq}\oplus b } \\
        & \; \subset \sigma\lrp{ \left[-\lrp{2^{e_2+e_1 + 1}\oplus |\eta_n|^+}, 2^{e_2+e_1 + 1}\oplus |\eta_n|^+\right]_{\fpq}}.
    \end{align*}
    This completes the proof.

\subsection{Proof of \cref{lemma:sigmoidal_distinguish}}\label{sec:pflem:sigmoidal_distinguish}

Suppose $\sigma$ has two separating points $\eta_1, \eta_2$ with  $\expo{\eta_1}\le 1$ and $\expo{\eta_2}\ge 2$. 
We apply \cref{lemma:distinguishing_points_to_distinguishable} with $e_1=e_2 = \emax$. 
Since
\begin{equation*}
[\emin, \emax]_{\bbZ}\subset \bigcup_{i=1}^2 \left[\expo{\eta_i}-\emax, \expo{\eta_i}+\emax-2\right]_{\bbZ},
\end{equation*}
$\fpq = \left(-2^{\emax+1}, 2^{\emax+1}\right)_{\fpq}$ is $\sigma$-distinguishable with range
\begin{align*}
& \sigma\lrp{\left[ -\lrp{ 2^{e_1+e_2+1} \oplus |\eta_2|^+},  2^{e_1+e_2+1} \oplus |\eta_2|^+ \right]_{\fpq}} \\
 &\subset \sigma\lrp{\fpq\cup \{\infty,-\infty\}}
 \subset \left[-2^{\emax}, 2^{\emax}\right]_{\fpq}
\end{align*} 
under $\mcS$.

Suppose $\sigma$ has one separating point $\eta_1$ with  $\expo{\eta_1}=1$. Then we have 
\begin{equation*}
[\emin, \emax-1]_{\bbZ}\subset \left[\expo{\eta_1}-\emax, \expo{\eta_1}+\emax-2\right]_{\bbZ}.
\end{equation*}
We apply \cref{lemma:distinguishing_points_to_distinguishable} with $e_1 = \emax , e_2 = \emax - 1$. Then $(-2^{\emax}, 2^{\emax})_\fpq$ is $\sigma$-distinguishable with range 
\begin{align*}
& \sigma\lrp{\left[ -\lrp{ 2^{e_1+e_2+1} \oplus |\eta_2|^+},  2^{e_1+e_2+1} \oplus |\eta_2|^+ \right]_{\fpq}} \\
 &\subset \sigma\lrp{\fpq\cup \{\infty,-\infty\}}
 \subset \left[-2^{\emax}, 2^{\emax}\right]_{\fpq}
\end{align*} 
under $\mcS$.

\subsection{Proof of \cref{cor:sigmoidal_distinguish}}\label{sec:pflem:sigmoidal_distinguish}

Note that we assume $\mbit \ge 2$.

First consider $\sigma = \round{\rho}$, $\rho =\Sigmoid= \frac{1}{1+e^{-x}}$. 
Note that $\sigma$ is increasing and 
\[ \rho(0)= \tfrac 12, \;  \rho(1) \ge \tfrac 5 8, \; \rho(-4) \ge 2^{-6}, \; \rho(-5) \le 2^{-7}.   \]
Hence $\sigma$ has two separating points $\eta_1, \eta_2 \in \fpq$ such that $0 \le \eta_1 < 1$ and $ -5 \le \eta_2  < -4$. Since $\sigma(\fpq) \subset [0,1]_\fpq$, by \cref{lemma:sigmoidal_distinguish}, $\fpq$ is $\sigma$-distinguishable with range $[-2^{\emax}, 2^{\emax}]_\fpq$ under $\mcS$. By \cref{thm:main}, there exists a four-layer $\sigma$ network under $\mcS$ such that $f$ can represent given function $g : \domain \to \fpq \cup \{ -\infty, \infty\}$. 

Next consider $\sigma = \round{\rho}$ for $\rho = \tanh$. 
Note that $\sigma$ is increasing and 
\[ \rho(0)= 0, \;  \rho(1) \ge \tfrac 3 4 , \; \rho(4) < 1-2^{-11}, \; \rho(6) > 1 - 2^{-14}.   \]
If $\mbit \ge 9$, $\sigma = \round{\rho}$ has two separating points $\eta_1,\eta_2 \in \fpq$ such that $0 \le \eta_1 < 1$ and $4 \le \eta_2 < 6$.  \\
For $2 \le \mbit \le 8$, $\sigma$ has one separating point $\eta_1 \in \fpq$ such that $0 \le \eta_1 < 1$. Since $\sigma(\fpq) \subset [-1,1]_\fpq$ by \cref{lemma:sigmoidal_distinguish}, $\mcK$ is $\sigma$-distinguishable with range $[-2^{\emax}, 2^{\emax}]_\fpq$ under $\mcS$ where 
\begin{align*}
    \mcK = \begin{cases}
            \fpq \; &\text{if} \quad \mbit \ge 9 \\ 
            (-2^{\emax}, 2^{\emax})_{\fpq} \; &\text{if} \quad 2 \le \mbit \le 8    \end{cases}
\end{align*}

By \cref{thm:main}, there exists a four-layer $\sigma$ network under $\mcS$ such that $f$ can represent given function $g : \domain \to \fpq \cup \{ -\infty, \infty\}$ where
\begin{align*}
    \domain = \begin{cases}
            \fpq^d \; &\text{if} \quad \mbit \ge 9 \\ 
            (-2^{\emax}, 2^{\emax})_{\fpq}^d \; &\text{if} \quad 2 \le \mbit \le 8    \end{cases}
\end{align*}

\subsection{Proof of \cref{lemma:single_distinguishing_point}}\label{sec:pflem:single_distinguishing_point}
Without loss of generality, consider the case: $\resig'(x) \ge L$ for all $x\in [a,b]$ and $\rho(a)\ge 0$. Note that $\expo{\sigma(a)} , \expo{\sigma(b)} \le e $.
Since 
\begin{equation*}
   \resig\lrp{b} - \resig\lrp{a}
    \ge L (b-a) \ge 2^{e-\mbit},
\end{equation*}
$\sigma(b)\neq \sigma(a)$. 
Therefore, there exists $\eta\in [a,b)_{\fpq}$ such that
$\sigma(\eta^-) < \sigma(\eta)\le \sigma(\eta^+)$.
This completes the proof.

\subsection{Proof of \cref{lemma:real_function_distinguishing}}\label{sec:pflem:real_function_distinguishing}
Define $\sigma $ as $\sigma\defeq \round{\hat\rho}$.
We will prove that for any $e_* \in [\emin, e-1]_{\bbZ}$, there exists a separating point $\eta$ such that $\expo{\eta}=e_*$.
For any floating-point number $x\in \fpq$, if $\sigma(x)$ is normal, as 
\begin{equation*}
    |\sigma(x)| = \left|\round{\hat\resig(x)}\right| \le L_2 x \times \lrp{1+2^{-\mbit}},
\end{equation*}
$\expo{\sigma(x)}\le \expo{x}+ \left\lceil \log_2\lrp{L_2\times \lrp{1+2^{-\mbit}}}\right\rceil_{\bbZ}$.
    Note that for any $x\ge 2^{\emin}$, 
\begin{equation*}
     \sigma(x)\ge \sigma\lrp{2^{\emin}}
    \ge 2^{l_1} 2^{\emin} \ge \fmin, 
\end{equation*}
Consider \cref{lemma:single_distinguishing_point} with $a = 2^{e_*}$, $b = \lrp{2^{e_{*}+1}}^-$, $L=L_1$, and $e = \expo{\sigma(x)}$.
As \begin{align*}
   &L_1\lrp{\lrp{2^{e_*+1}}^- - 2^{e_*}} =  L_1 \lrp{1-2^{-\mbit}} 2^{e_*} \\
 &\ge  L_1 \lrp{1-2^{-\mbit}}2^{\expo{\sigma(x)} - \left\lceil \log_2\lrp{L_2\times \lrp{1+2^{-\mbit}}}\right\rceil_{\bbZ}} \\
 &\ge 2^{\expo{\sigma(x)}-\mbit}, 
\end{align*}
  there exists a separating point $\eta\in [2^{e_*}, \lrp{2^{e_*+1}}^-)_\fpq$ such that $\expo{\eta}=e_*$.
    Then, by \cref{lemma:distinguishing_points_to_distinguishable} with separating points with exponents $\emin, \emin+1, \dots, e-1 $,
    $e_1 = 0$, and $e_2 = \min\lrp{\emax-l_2-1, \emax + e -3}$, the interval $\lrp{-2^{e_2+1},2^{e_2+1}}_{\fpq}$ is $\sigma$-distinguishable with range 
    \begin{align*}
    &\sigma\lrp{\left[ -\lrp{ 2^{e_1+e_2+1} \oplus |\eta_n|^+},  2^{e_1+e_2+1} \oplus |\eta_n|^+ \right]} \\ 
    &\subset 
    L_2\lrp{1+2^{-\mbit}}\left[ -\lrp{ 2^{e_1+e_2+1} \oplus |\eta_n|^+},  2^{e_1+e_2+1} \oplus |\eta_n|^+ \right]
    \\ &\subset 2^{l_2} \left[-\lrp{2^{\emax-l_2}}, 2^{\emax-l_2}\right]
    \subset \left[-2^{\emax}, 2^{\emax}\right],
\end{align*}
under $\mcS$.
This completes the proof.
\subsection{Proof of \cref{cor:real_function_distinguishing}}
\label{sec:pflem::real_function_distinguishing}

Let $\hat{\rho}(x)$ $= \rho(x)$ in \cref{lemma:real_function_distinguishing}. 
By \cref{lemma:real_function_distinguishing} and \cref{table:lip}, $(-2^{\emax-2}, 2^{\emax-2})_{\fpq}$ is $\round{\rho}$-distinguishable with range $\left[-2^{\emax}, \allowbreak 2^{\emax}\right]_{\fpq}$ under $\mcS$. By \cref{thm:main}, we complete the proof.

\begin{table}[h]
\centering
    \begin{tabular}{cccccc@{\;\;\;}ccccc}
\toprule
\begin{tabular}{@{}c@{}}
  $\rho(x)$ 
\end{tabular}
& $e$ & $L_1$ & $L_2$ & $l_1$ & $l_2$ & $e'$   \\ 
\midrule
$\mathrm{Identity}$   & $\emax$ & 1    & 1       & -1             & 1 & $\emax-1$     \\
$\relu$   & $\emax$ & 1    & 1       & -1             & 1 & $\emax-1$     \\
$\elu$    & $\emax$ & 1    & 1       & -1             & 1 & $\emax-1$     \\
$\GeLU$   & $\emax$ & 0.5 & 1       & -2 & 1 & $\emax-1$ \\
$\text{SeLU}$   & $\emax$ & 1.05 & 1.05    & -1 & 1 & $\emax-1$ \\
$\Swish$   & $\emax$ & 0.5 & 1   & -2 & 1 & $\emax-1$  \\
$\Mish$   & $\emax$ & 0.6 & 1       & -2 & 1 & $\emax-1$  \\
$\sin$  & 0 & 0.540 & 1  & -2 & 1 & $\emax-2$ \\
\bottomrule
\end{tabular}
\vspace{0.2in}
\caption{%
  Properties of floating-point format for verifying the conditions.
  Numbers in the table are rounded to the second decimal place.
}
\label{table:lip}
\end{table}

\section{Proofs of theorems and lemmas on sufficient conditions of non-correctly rounded activation functions}

\subsection{Proof of \cref{cor:sigmoidal_distinguish_nc}}\label{sec:pflem:sigmoidal_distinguish_nc}

Note that we assume $\mbit \ge 7$.

First consider $\rho =\Sigmoid= \frac{1}{1+e^{-x}}$. 
Note that $\sigma$ is increasing and 
\[ \rho(0)= \tfrac 12, \;  \rho(\tfrac 3 2) \ge \tfrac 3 4, \; \rho(-4) \ge 2^{-6}, \; \rho(-5) \le 2^{-7}.   \]
Since $\sigma$ approximates $\rho$ with $ K\le 10$ ulp errors, we have 
\begin{align*}
    \sigma(0) &\le \tfrac 1 2 ( 1 + K \cdot 2^{-\mbit} ) \le \tfrac{9}{16}, \\
    \sigma(\tfrac 3 2) &\ge \tfrac 3 4 - 2^{-1}  (  K \cdot 2^{-\mbit} )  \ge \tfrac{11}{16}, \\
    \sigma(-4) & \ge 2^{-6} ( 1 - K \cdot 2^{-\mbit} ) \ge \tfrac{7}{4}  \cdot 2^{-7} \\
        \sigma(-5) & \le 2^{-7} ( 1 + K \cdot 2^{-\mbit} ) \le \tfrac{9}{8}  \cdot 2^{-7}.
\end{align*}
Hence $\sigma$ has a separating point $0 \le \eta_1 < \tfrac 3 2, -5 \le \eta_2 < -4$.
 Since $\sigma(\fpq) \subset [0,1]_\fpq$, by \cref{lemma:sigmoidal_distinguish}, $\fpq$ is $\sigma$-distinguishable with range $[-2^{\emax}, 2^{\emax}]_\fpq$ under $\mcS$. By \cref{thm:main}, there exists a four-layer $\sigma$ network under $\mcS$ such that $f$ can represent the given function $g : \domain \to \fpq \cup \{ -\infty, \infty\}$. 

Next, consider $\rho = \tanh$. 
Note that $\sigma$ is increasing and 
\[ \rho(0)= 0, \;  \rho(1) \ge \tfrac 3 4 , \; \rho(4) < 1-2^{-11}, \; \rho(6) > 1 - 2^{-14}.   \]
Since $\sigma$ approximates $\rho$ with $ K\le 10$ ulp errors, we have 
\begin{align*}
    \sigma(0) &\le \tfrac 1 2 ( 1 + K \cdot 2^{-\mbit} ) \le \tfrac{9}{16}, \\
    \sigma(1) &\ge \tfrac 3 4 - 2^{-1}  (  K \cdot 2^{-\mbit} )  \ge \tfrac{11}{16}. 
\end{align*}
Hence $\sigma$ has one separating point $\eta_1 \in \fpq$ such that $0 \le \eta_1 < 1$. 
If $\mbit \ge 15$, 
\begin{align*}
    \sigma(4) &\le ( 1 - 2^{-11}) + 2^{-1} (K \cdot 2^{-\mbit}) \le 1 - 11 \cdot 2^{-15}, \\
    \sigma(5) &\ge ( 1 - 2^{-14}) -  2^{-1} (K \cdot 2^{-\mbit}) \ge 1 - 7 \cdot 2^{-15}, \\
\end{align*}
$\sigma$ has two separating points $\eta_1,\eta_2 \in \fpq$ such that $0 \le \eta_1 < 1$ and $4 \le \eta_2 < 6$.
Since $\sigma(\fpq) \subset [-1,1]_\fpq$ by \cref{lemma:sigmoidal_distinguish}, $\mcK$ is $\sigma$-distinguishable with range $[-2^{\emax}, 2^{\emax}]_\fpq$ under $\mcS$ where 
\begin{align*}
    \mcK = \begin{cases}
            \fpq \; &\text{if} \quad \mbit \ge 15 \\ 
            (-2^{\emax}, 2^{\emax})_{\fpq} \; &\text{if} \quad 7 \le \mbit \le 14    \end{cases}
\end{align*}

By \cref{thm:main}, there exists a four-layer $\sigma$ network under $\mcS$ such that $f$ can represent given function $g : \domain \to \fpq \cup \{ -\infty, \infty\}$ where
\begin{align*}
    \domain = \begin{cases}
            \fpq^d \; &\text{if} \quad \mbit \ge 15\\ 
            (-2^{\emax}, 2^{\emax})_{\fpq}^d \; &\text{if} \quad 7 \le \mbit \le 14    \end{cases}
\end{align*}

\subsection{Proof of \cref{lemma:single_distinguishing_point_impl}}\label{sec:pflem:single_distinguishing_point_impl}

Without loss of generality, consider the case: $\resig'(x) \ge L$ for all $x\in [a,b]_\fpq$ and $\rho(a)\ge 0$.

Since
\begin{align*}
\resig\lrp{b} - \resig\lrp{a}
\ge L (b-a) \ge 2^{e-\mbit + 5 }, 
\end{align*}
and
\begin{align*}
|\sigma(x) - \rho(x)| &\le K \cdot \mathrm{ulp}(\rho(x))\le K \cdot 2^{e-\mbit},
\end{align*}
for $x \in [a,b]_\fpq$,

we have 
\begin{align*}
&|\sigma(b) - \sigma(a) | \\
&\ge 
|\rho(b) - \rho(a)| - |\sigma(b) - \rho(b) | - |\sigma(b) - \rho(b) |  \\
&\ge 2^{e+5-\mbit} - 20 \cdot 2^{e-\mbit} \ge 2^{e-\mbit},
\end{align*}

    Therefore, there exists $\eta\in [a,b]_{\fpq}$ such that
    $$\sigma(\eta^-) < \sigma(\eta)\le \sigma(\eta^+).$$
    This completes the proof.

\subsection{Proof of \cref{lemma:real_function_distinguishing_ncrounded}}\label{sec:pflem:real_function_distinguishing_ncrounded}
First note that $\mbit  \ge 7$.
We will prove that for any $ e_* \in [\emin, 2]_{\bbZ}$, there exists a separating point $\eta$ such that $\expo{\eta}=e_*$.

We first suppose $\tfrac{3}{10} \le L_1 \le L_2 \le 1$. 

Let $a = 2^{e_*}$, $b =  (2-2^{1-\mbit}) \times 2^{e_*}$, $L=L_1$, and $e = e_*$ in \cref{lemma:single_distinguishing_point_impl}.
We have 
\begin{align*}
    |\rho(a)|, |\rho(b)| \le L_2  b \le (2-2^{1-\mbit} ) \times 2^{e_*} \le (1-2^{-\mbit}) \cdot 2^{e_*+1}. 
\end{align*}

Since 
\[ (1-2^{1-\mbit})L_1 \ge (1-2^{6}) \cdot \tfrac{3}{10} \approx 0.2953, \]
we have 
\begin{align*}
    L_1 (b-a) \ge (1-2^{1-\mbit}) L_1  2^{e_*} \ge 2^{-2} \cdot 2^{e_*}\ge 2^{e_*-\mbit+5}.
\end{align*}
By \cref{lemma:single_distinguishing_point_impl}. we have  a separating point $\eta\in [2^{e_*}, 2^{e_*+1})_\fpq$ such that $\expo{\eta}=e_*$.

Next, assume $ \tfrac{3}{5} \le L_1 \le L_2 \le 2$. Let $a = 2^{e_*}$, $b =  (2-2^{1-\mbit}) \times 2^{e_*}$, $L=L_1$, and $e = e_*+1$ in \cref{lemma:single_distinguishing_point_impl}.
We have 
\begin{align*}
    |\rho(a)|, |\rho(b)| \le L_2  b \le (2-2^{1-\mbit} ) \times 2^{e_*+1} \le (1-2^{-\mbit}) \cdot 2^{e_*+2}. 
\end{align*}
Since 
\[ (1-2^{1-\mbit})L_1 \ge (1-2^{6}) \cdot \tfrac{3}{5} \approx 0.5906, \]
we have 
\begin{align*}
    L_1 (b-a) \ge (1-2^{1-\mbit}) L_1  2^{e_*} \ge 2^{-1} \cdot 2^{e_*}\ge 2^{e_*-\mbit+6}.
\end{align*}
By \cref{lemma:single_distinguishing_point_impl}. we have a separating point $\eta\in [2^{e_*}, 2^{e_*+1})_\fpq$ such that $\expo{\eta}=e_*$.

In both cases, we have a separating point $\eta\in [2^{e_*}, 2^{e_*+1})_\fpq$ such that $\expo{\eta}=e_*$.

Therefore, we have separating points $\eta_1 , \eta_2 , \dots , \eta_n$ with exponents \[\emin, \emin+1, \dots, 2,\] 
Let $e_1=0$. We have 
\[   \bigcup_{i=1}^n [\expo{\eta_i} - e_1 , \expo{\eta_i} + \emax - 2]_\bbZ = [\emin, \emax].\]
Then, by \cref{lemma:distinguishing_points_to_distinguishable} the interval $\lrp{-2^{e_2+1},2^{e_2+1}}_{\fpq}$ is $\sigma$-distinguishable with range 
\begin{align*}
    &\sigma\lrp{\left[ -\lrp{ 2^{e_1+e_2+1} \oplus |\eta_n|^+},  2^{e_1+e_2+1} \oplus |\eta_n|^+ \right]} \\ 
    \subset 
    & \; \sigma  \left( \left[ -\lrp{ 2^{e_2+1} \oplus 2^3},  2^{e_2+1} \oplus 2^3 \right] \right),
\end{align*}
under $\mcS$.

Let $e_2=\emax-2-k$.
Since $L_2 \le 2^{\emax-\mbit-5}$, we have 
\begin{align*}
     L_2 &\le 2^{\emax-\mbit-5} \\
    k &\le \emax -\mbit -5,  \\
    2^{e_2+1} &= 2^{\emax -1 - k} \ge 2^{ \mbit + 4}
\end{align*}
leading to 
\[ 2^{e_2 +1} \oplus 2^3 =  2^{e_2+1}. \]

By \cref{cond:impl}, for $ |x| \ge 1$, we have 
$$ |\sigma(x)| \le 2 |\rho(x) |.$$ 
Therefore, 
\begin{align*}
     \sigma  &\left( \left[ -\lrp{  2^{e_2+1} \oplus 2^e},  2^{e_2+1} \oplus 2^e  \right] \right) \\
     =   \sigma  &\left( \left[ -\lrp{  2^{e_2+1}},  2^{e_2+1}   \right] \right) \\
     \subset 2 \cdot \rho  &\left( \left[ -\lrp{  2^{e_2+1}},  2^{e_2+1}   \right] \right) \\
     \subset  & \left[ - \lrp{ L_2 \cdot 2^{e_2+2}},   L_2  \cdot 2^{e_2+2}   \right]  \\
      \subset  & \left[ - \lrp{  2^{e_2+2+k}},  \cdot 2^{e_2+2+k}   \right] 
\end{align*}
Hence $(- 2^{\emax-1-k}, 2^{\emax-1-k} )_\fpq $ is $\sigma$-distinguishable with range $\left[-2^{\emax}, 2^{\emax}\right]_\fpq$.

This completes the proof.

\subsection{Proof of \cref{cor:real_function_distinguishing_ncrounded}}\label{sec:pfcor:real_function_distinguishing_ncrounded}
Let $\hat{\rho}(x)$ $= \rho(x)$ in \cref{lemma:real_function_distinguishing_ncrounded}. By \cref{table:lip}, we can check the assumptions of \cref{lemma:real_function_distinguishing_ncrounded}.
\\ We have 
$L_2\le2$ and $k\le 1$  for $\mathrm{Identity},\relu,\elu,\GeLU,\Swish,\Mish,\sin, \mathrm{SeLU}$.  
Hence $(-2^{\emax-2}, 2^{\emax-2})_{\fpq}$ is $\sigma$-distinguishable with range $\left[2^{\emax}, \allowbreak 2^{\emax}\right]_{\fpq}$ under $\mcS$.  By \cref{thm:main}, we complete the proof.